\documentclass[11pt]{article}

\usepackage[a4paper,top=3cm,bottom=3cm,left=3cm,right=3cm,marginparwidth=1.75cm]{geometry}
\usepackage{amsmath,amssymb,amsthm}
\usepackage{booktabs}
\usepackage{natbib}
\usepackage{algorithm}
\usepackage{algorithmic}
\usepackage{graphicx}
\usepackage{subcaption}
\usepackage{caption}
\usepackage{enumerate}
\usepackage{float}
\usepackage{authblk}
\usepackage{xcolor}

\usepackage[hidelinks]{hyperref}
\usepackage[nameinlink,noabbrev]{cleveref}

\newtheorem{theorem}{Theorem}

\newtheorem{lemma}{Lemma}
\newtheorem{corollary}{Corollary}

\newtheorem{assumption}{Assumption}

\newcommand{\E}{\mathbb{E}}
\renewcommand{\P}{\mathbb{P}}

\newcommand{\N}{\mathcal{N}}
\newcommand{\X}{\mathcal{X}}
\newcommand{\F}{\mathcal{F}}

\newcommand{\Var}{\operatorname{Var}}

\title{Gaussian Process Bandits with Paired Auxiliary Observations}

\author{%
\parbox{\textwidth}{\centering
Weibo Huang\textsuperscript{1,*} \quad
Qing Feng\textsuperscript{2,*} \quad
Ruihao Zhu\textsuperscript{3,\(\dagger\)} \quad
Cheng Hua\textsuperscript{1,\(\dagger\)}
\par\smallskip
{\footnotesize
\textsuperscript{1} Antai College of Economics and Management, Shanghai Jiao Tong University\par
\textsuperscript{2} School of Operations Research and Information Engineering, Cornell University\par
\textsuperscript{3} SC Johnson College of Business, Cornell University\par
\smallskip
3338725059@sjtu.edu.cn \quad
qf48@cornell.edu \quad
ruihao.zhu@cornell.edu \quad
cheng.hua@sjtu.edu.cn\par
}
}
}

\date{}

\begin{document}

\maketitle

\begin{abstract}
Bayesian optimization (BO) is widely used for expensive black-box optimization, yet observation noise makes each evaluation less informative and slows posterior learning. Standard BO learns from costly outcomes gathered during optimization. In many applications, however, low-cost auxiliary data are available beforehand. We introduce Auxiliary-Informed Gaussian Process Bandit (AI-GPB), a variance-reduction method that combines an offline estimate of the auxiliary mean with a correction coefficient learned from paired online observations. Notably, it exploits  correlation between paired observation errors without requiring the auxiliary mean to approximate the objective. We establish finite-sample GP confidence bounds and regret guarantees, including explicit cumulative and simple regret rates for Matérn kernels. Numerical experiments across continuous BO benchmarks with varied landscapes and dimensions show consistent reductions in cumulative and simple regret relative to vanilla GP-UCB, even when within-pair correlation is moderate and the auxiliary mean is estimated from noisy offline data.
\end{abstract}

\section{Introduction}
Black-box optimization with costly, noisy evaluations is widespread across science and engineering. A common approach in Bayesian optimization (BO) uses a Gaussian process (GP) posterior and an acquisition function to select evaluations sequentially. A widely used acquisition rule is GP-UCB, which balances exploration and exploitation through the posterior mean and uncertainty \citep{srinivas2009gaussian}.  However, observation noise limits the information gained from each evaluation, thereby slowing posterior learning and reducing the effectiveness of subsequent queries.

Paired auxiliary observations offer a way to reduce this information loss. For example, a production run may require a costly quality test while generating routine process data \citep{dreyfus2022virtual}.  In chemistry, an experiment may provide product yield together with inexpensive colour measurements from the same reaction \citep{leonov2024integrated}.  Historical records may also contain such data without corresponding quality or yield measurements.

Classical control variates use correlation between paired observations to reduce variance, but generally assume that the auxiliary mean is known \citep{lavenberg1981perspective,nelson1990control}. Related bandit formulations use a finite action set and repeated arm pulls, while BO with auxiliary information generally relies on predictive relationships between response functions \citep{ji2025multi,kandasamy2016gaussian}. Two challenges remain in continuous GP bandits. First, the auxiliary mean must be estimated from finite, noisy offline data, introducing a spatially varying bias into the function learned by the GP. Second, continuous BO rarely revisits the same point, so the auxiliary weight must be learned from observations along the adaptive query path.

We propose Auxiliary-Informed Gaussian Process Bandit (AI-GPB), a two-stage algorithm that combines offline auxiliary data with paired online outcomes to reduce observation variance. In the offline phase, we estimate the auxiliary mean function to center the online auxiliary output, not to predict the objective. In the online phase, we estimate the correction coefficient from the paired observations and update the GP posterior with the corrected costly outcomes. Figure~\ref{fig:paired-correction-illustration} illustrates how the resulting noise reduction increases posterior precision and can change the next GP-UCB query. We summarize the main contributions of this paper as follows.

\begin{figure}[t]
    \centering
    \includegraphics[width=\textwidth]{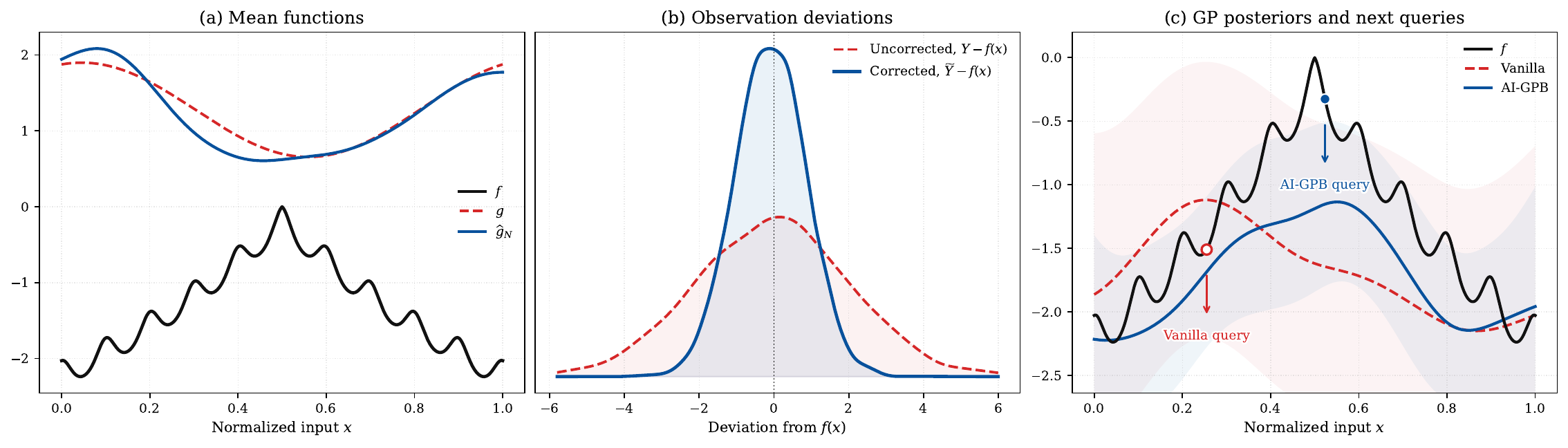}
    \caption{\textbf{Mechanism of AI-GPB.}
    The example uses a one-dimensional cross-section of the Ackley function with \(\rho=0.9\).  Panel (a) shows \(f\), \(g\), and the offline estimate \(\widehat g_N\). Panel (b) compares the uncorrected and corrected observation errors. Panel (c) shows the GP posteriors and next queries of Vanilla GP-UCB and AI-GPB after the same initial queries. Shading denotes 90\% posterior bands.}
    \label{fig:paired-correction-illustration}
\end{figure}

\paragraph{Paired variance reduction.}
We formulate a continuous GP bandit framework in which the auxiliary mean function is estimated from noisy offline data and the correction coefficient is learned from paired online observations. The framework requires no predictive relationship between the auxiliary mean and the objective and explicitly accounts for offline estimation error across the search space.

\paragraph{Regret guarantees with learned correction.}
We establish GP confidence bounds and cumulative regret guarantees for a stabilized version of AI-GPB. The analysis quantifies how offline estimation error affects coefficient learning under adaptive sampling. For Matérn kernels, we derive explicit rates for cumulative and simple regret. We further identify a sufficient scaling of the offline sample size under which the leading oracle gain from variance reduction is retained.

\paragraph{Empirical gains in continuous BO.}
Across benchmarks with varied landscapes and dimensions, AI-GPB reduces both cumulative and simple regret, with generally larger gains at higher levels of within-pair correlation. These improvements persist when the auxiliary mean is learned from noisy offline data and constructed independently of the objective. A controlled study of the offline budget further shows how improved estimation of \(g\) narrows the gap to the known-\(g\) reference.

\section{Related Work}
Our work connects variance reduction from paired stochastic outputs with Gaussian process bandit optimization.  We first review paired-output variance reduction in simulation and bandit settings, and then discuss GP bandits and Bayesian optimization with auxiliary information.

\paragraph{Paired simulation outputs.}
In stochastic simulation, the output depends on a decision \(x\) and an exogenous random input \(\xi\). Common-random-number (CRN) methods use the same realizations of \(\xi\) across different decisions, thereby reducing the variance of estimated performance differences \citep{kleijnen1975antithetic,pearce2022bayesian}. Control variates use a correlated auxiliary variable with known mean to reduce estimator variance \citep{lavenberg1981perspective,nelson1990control}. Multifidelity methods extend this idea to expensive and inexpensive simulators evaluated with shared random inputs \citep{peherstorfer2016optimal,gorodetsky2020generalized}. We study paired expensive and auxiliary outputs in GP bandit optimization and use their dependence to reduce observation noise.

\paragraph{Variance reduction in bandits.} 
Related variance-reduction ideas have been applied to sequential bandit problems. Verma and Hanawal apply control variates with known auxiliary means to stochastic multi-armed bandits \citep{verma2021stochastic}.  MLA-UCB adapts prediction-powered inference to finite-arm bandits using ML-generated surrogates observed offline and online \citep{angelopoulos2023prediction,angelopoulos2023ppi++,ji2025multi}. Verma et al. exploit correlated auxiliary feedback to obtain tighter confidence bounds in parameterized bandits \citep{verma2023exploiting}. In our setting, the auxiliary mean is learned over a continuous domain from noisy offline observations, and we analyze how its spatial error affects online coefficient learning, GP confidence bounds and regret.

\paragraph{Gaussian process bandits.}
Gaussian process bandits use kernel structure to learn an unknown reward function over a large or continuous domain.  Because each observation changes uncertainty beyond the queried location, reducing observation noise can also alter subsequent queries.  GP-UCB selects queries using kernel confidence bounds, and its cumulative-regret guarantees depend on both the kernel and the observation-noise level
\citep{srinivas2009gaussian,chowdhury2017kernelized,
vakili2021information,iwazaki2025improved}.  We build on this framework to analyze GP bandits with paired auxiliary observations.

\paragraph{Bayesian optimization with auxiliary information.}
Bayesian optimization can use auxiliary data by modeling relationships between response functions. Multifidelity and multi-information-source methods jointly model the objective and cheaper response functions, while transfer and multi-task BO use related tasks or offline simulators to inform optimization of the target function \citep{kandasamy2016gaussian,poloczek2017multi, letham2019bayesian,wang2024pre}.  More recently, PA-GP-UCB models the objective and an ML prediction function as correlated outputs of a multi-task GP while assuming independent observation noises \citep{chen2026gaussian}. By comparison, AI-GPB obtains variance reduction from correlation between paired observation errors and does not require the auxiliary mean to approximate the objective.

\section{Problem Setup}

\subsection{Paired Observations}
\label{subsec:paired-observations}

Let \(\X\subset\mathbb R^d\) be a compact domain, and let \(\xi\sim P_\xi\) denote a random input.  For each decision \(x\in\X\), a joint observation mechanism generates an expensive outcome \(Y(x,\xi)\) and a cheap auxiliary observation \(Z(x,\xi)\). The two components are simultaneously generated and observed as a pair during online optimization.

Their mean functions are
\begin{equation}
\label{eq:def-mean-functions}
\begin{aligned}
    f(x)
    &:=
    \E_\xi\!\left[Y(x,\xi)\right],\\
    g(x)
    &:=
    \E_\xi\!\left[Z(x,\xi)\right].
\end{aligned}
\end{equation}
The objective is to maximize the expensive mean response \(f\) over \(\X\). The auxiliary mean \(g\) is not itself an optimization target.

We define the centered stochastic fluctuations by
\begin{equation}
\label{eq:def-centered-fluctuations}
\begin{aligned}
    \varepsilon^Y(x,\xi)
    &:=
    Y(x,\xi)-f(x),\\
    \varepsilon^Z(x,\xi)
    &:=
    Z(x,\xi)-g(x).
\end{aligned}
\end{equation}
By construction, for every \(x\in\X\),
\begin{equation}
\label{eq:centered-fluctuations-mean-zero}
    \E_\xi\!\left[\varepsilon^Y(x,\xi)\right]
    =
    \E_\xi\!\left[\varepsilon^Z(x,\xi)\right]
    =
    0.
\end{equation}

Here, \(\xi\) denotes the joint random input underlying the paired observations. Their dependence will be characterized through the covariance of their centered fluctuations. For instance, \(Y\) may represent a high-fidelity simulation or a costly measured outcome, whereas \(Z\) may be a low-fidelity output, an ML-generated score, or an inexpensive auxiliary measurement associated with the same unit or scenario.

We do not require \(g\) to approximate \(f\) or share its maximizer. The relevant information is the dependence between the paired fluctuations \(\varepsilon^Y\) and \(\varepsilon^Z\).

\subsection{Offline Data and Estimation of \(g\)}
\label{subsec:offline-data}

Before online optimization begins, we have access to \(N\) auxiliary observations
\begin{equation}
\label{eq:def-offline-data}
    \mathcal D_{\mathrm{off}}
    :=
    \left\{
        \left(x_i^{\mathrm{off}},Z_i^{\mathrm{off}}\right)
    \right\}_{i=1}^N,
    \qquad
    Z_i^{\mathrm{off}}
    :=
    Z\!\left(x_i^{\mathrm{off}},\xi_i^{\mathrm{off}}\right),
    \qquad
    \xi_i^{\mathrm{off}}
    \stackrel{\mathrm{i.i.d.}}{\sim}
    P_\xi.
\end{equation}
Constructing \(\mathcal D_{\mathrm{off}}\) requires no expensive outcomes. The offline observations follow the same \(Z(x,\xi)\) mechanism and random-input distribution \(P_\xi\) as their online counterparts, while the offline random inputs are independent of those used online.

Using \(\mathcal D_{\mathrm{off}}\), an offline learning procedure constructs an estimate \(\widehat g_N:\X\to\mathbb R\) of \(g\).  The procedure may use, for example, a regression or another machine-learning model. Once constructed, \(\widehat g_N\) is held fixed throughout online optimization. We write
\begin{equation}
\label{eq:def-offline-g-error}
    \widehat g_N(x)
    =
    g(x)+\delta_N(x),
    \qquad x\in\X,
\end{equation}
where \(\delta_N(x)\) is the error in estimating the mean \(g(x)\), whereas \(\varepsilon^Z(x,\xi)\) is the stochastic fluctuation of a single observation around that mean.

\subsection{Online Paired Observations and Filtration}
\label{subsec:online-observations}

Let \(\F_0\) be the sigma-algebra generated by the offline dataset
\(\mathcal D_{\mathrm{off}}\). Both \(\widehat g_N\) and \(\delta_N\) are therefore \(\F_0\)-measurable.

At round \(t\geq 1\), the query \(x_t\in\X\) is selected based only on
\(\F_{t-1}\) and is therefore \(\F_{t-1}\)-measurable.  A fresh random input
\(\xi_t\sim P_\xi\), independent of \(\F_{t-1}\), is then realized, and the paired outcomes
\begin{equation}
\label{eq:def-online-paired-observations}
\begin{aligned}
    Y_t
    &:=
    Y(x_t,\xi_t)
    =
    f(x_t)+\varepsilon_t^Y,\\
    Z_t
    &:=
    Z(x_t,\xi_t)
    =
    g(x_t)+\varepsilon_t^Z
\end{aligned}
\end{equation}
are observed, where
\[
    \varepsilon_t^Y
    :=
    \varepsilon^Y(x_t,\xi_t),
    \qquad
    \varepsilon_t^Z
    :=
    \varepsilon^Z(x_t,\xi_t).
\]
After observing the pair, the information is updated as
\begin{equation}
\label{eq:def-online-filtration}
    \F_t
    :=
    \sigma\!\left(\F_{t-1},Y_t,Z_t\right).
\end{equation}
By the independence of \(\xi_t\) from \(\F_{t-1}\) and the centering in
\eqref{eq:def-centered-fluctuations},
\begin{equation}
\label{eq:online-conditional-centering}
    \E\!\left[\varepsilon_t^Y\mid\F_{t-1}\right]
    =
    \E\!\left[\varepsilon_t^Z\mid\F_{t-1}\right]
    =
    0.
\end{equation}
We center \(Z_t\) by its offline mean estimate and define the auxiliary residual
\begin{equation}
\label{eq:def-online-residual}
    \widetilde Z_t
    :=
    Z_t-\widehat g_N(x_t)
    =
    \varepsilon_t^Z-\delta_N(x_t).
\end{equation}
Here, \(\varepsilon_t^Z\) is the stochastic fluctuation at round \(t\), and \(\delta_N(x_t)\) is the \(\F_{t-1}\)-measurable offline estimation error. When \(\widehat g_N=g\), the estimation error vanishes and \(\widetilde Z_t=\varepsilon_t^Z\).

\subsection{Budget Structure and Performance Measures}
\label{subsec:budget-performance}

We distinguish the offline auxiliary sample size \(N\) from the online horizon \(T\). Auxiliary observations can often be collected in large numbers before optimization because they do not require expensive outcomes. Each online round produces a pair \((Y_t,Z_t)\). Observing \(Z_t\) adds little to the cost
of observing \(Y_t\). Regret is accumulated over the \(T\) online rounds. Let \(x^\star\in\arg\max_{x\in\X}f(x)\) be a maximizer of \(f\).  The instantaneous regret at round \(t\) is
\begin{equation}
\label{eq:def-instantaneous-regret}
    r_t
    :=
    f(x^\star)-f(x_t).
\end{equation}
Our primary performance measure is the cumulative regret
\begin{equation}
\label{eq:def-cumulative-regret}
    R_T
    :=
    \sum_{t=1}^T r_t.
\end{equation}
The simple regret after \(t\) rounds is
\begin{equation}
\label{eq:def-simple-regret}
    S_t
    :=
    \min_{1\leq s\leq t}r_s
    =
    f(x^\star)-\max_{1\leq s\leq t}f(x_s),
    \qquad t=1,\ldots,T.
\end{equation}
It measures the quality of the best point queried by round \(t\). By construction, \(S_T\leq R_T/T\).  These regret quantities are defined using the latent mean function \(f\) and are not part of the information available to the algorithm.

\subsection{Standing Assumptions}
\label{subsec:standing-assumptions}

The main analysis is developed under the following assumptions. Additional conditions are stated with the results that require them.

\begin{assumption}[Kernel and objective regularity]
\label{assumption:kernel-objective}
Let \(k:\X\times\X\to\mathbb R\) be a continuous positive definite kernel normalized so that
\[
    k(x,x)\leq 1,
    \qquad x\in\X.
\]
The objective \(f\) belongs to the associated reproducing kernel Hilbert space
\(\mathcal H_k\) and satisfies
\[
    \|f\|_{\mathcal H_k}\leq B
\]
for some \(B>0\).
\end{assumption}

\begin{assumption}[Paired stochastic fluctuations]
\label{assumption:paired-fluctuations}
For every round \(t\geq 1\), conditional on \(\F_{t-1}\), the paired
fluctuations satisfy
\begin{equation}
\label{eq:paired-noise-distribution}
    \begin{pmatrix}
        \varepsilon_t^Y\\
        \varepsilon_t^Z
    \end{pmatrix}
    \sim
    \N\!\left(
        \begin{pmatrix}
            0\\
            0
        \end{pmatrix},
        \Sigma
    \right),
    \qquad
    \Sigma
    :=
    \begin{pmatrix}
        \sigma_Y^2 & \sigma_{YZ}\\
        \sigma_{YZ} & \sigma_Z^2
    \end{pmatrix}
    \succ 0.
\end{equation}
The covariance matrix \(\Sigma\) is fixed across query locations and online rounds.
\end{assumption}
The covariance \(\sigma_{YZ}\) captures the dependence within each observed pair. We denote the corresponding correlation by
\begin{equation}
\label{eq:def-paired-correlation}
    \rho
    :=
    \frac{\sigma_{YZ}}{\sigma_Y\sigma_Z}.
\end{equation}

\begin{assumption}[Offline estimation accuracy]
\label{assumption:offline-estimation}
There exist deterministic quantities \(e_N\geq 0\) and \(\alpha_N\in[0,1)\) such that, with probability at least \(1-\alpha_N\) over the offline dataset,
\begin{equation}
\label{eq:offline-uniform-error}
    \|\widehat g_N-g\|_\infty
    =
    \|\delta_N\|_\infty
    :=
    \sup_{x\in\X}|\delta_N(x)|
    \leq e_N.
\end{equation}
\end{assumption}

Assumption~\ref{assumption:offline-estimation} controls the offline estimation error without requiring \(\widehat g_N\) to be unbiased for \(g\). The next section constructs an observation-level correction using \(\widetilde Z_t\) and incorporates it into GP-UCB.

\section{Method}

\subsection{Paired Correction}
\label{subsec:paired-correction}

We first consider the reference case in which \(\widehat g_N=g\). The
auxiliary residual then satisfies
\[
    \widetilde Z_t
    =
    Z_t-g(x_t)
    =
    \varepsilon_t^Z.
\]
Let \(\lambda\) be a coefficient chosen before observing the pair at round \(t\). We define the corrected outcome as
\begin{equation}
\label{eq:def-corrected-outcome-known-g}
    \widetilde Y_t(\lambda)
    :=
    Y_t-\lambda\widetilde Z_t
    =
    f(x_t)+\varepsilon_t^Y-\lambda\varepsilon_t^Z.
\end{equation}
Its conditional mean is \(f(x_t)\), and its conditional variance is
\begin{equation}
\label{eq:def-corrected-noise-variance}
    \eta_\lambda^2
    :=
    \Var\!\left(
        \widetilde Y_t(\lambda)
        \mid \F_{t-1}
    \right)
    =
    \sigma_Y^2
    -2\lambda\sigma_{YZ}
    +\lambda^2\sigma_Z^2.
\end{equation}
The variance is minimized by
\begin{equation}
\label{eq:def-oracle-coefficient}
    \lambda^\star
    :=
    \frac{\sigma_{YZ}}{\sigma_Z^2}.
\end{equation}
The resulting variance is
\begin{equation}
\label{eq:def-oracle-variance}
    \eta_\star^2
    :=
    \eta_{\lambda^\star}^2
    =
    \sigma_Y^2
    -
    \frac{\sigma_{YZ}^2}{\sigma_Z^2}
    =
    \sigma_Y^2(1-\rho^2).
\end{equation}
Equations~\eqref{eq:def-corrected-outcome-known-g}--%
\eqref{eq:def-oracle-variance} give the classical control-variate correction \citep{nelson1990control}. Here, \(\widetilde Z_t\) is the zero-mean control and \(\lambda^\star\) is the variance-minimizing coefficient.  The correction uses dependence between the paired fluctuations and does not require \(g\) to approximate \(f\).

In practice, the unknown \(g(x)\) in the auxiliary residual is replaced by the offline estimate \(\widehat g_N(x)=g(x)+\delta_N(x)\). Using
\(\widetilde Z_t=Z_t-\widehat g_N(x_t)\), the corrected outcome becomes
\begin{equation}
\label{eq:def-corrected-outcome-estimated-g}
\begin{aligned}
    \widetilde Y_t(\lambda)
    &=
    Y_t-\lambda\widetilde Z_t\\
    &=
    f(x_t)
    +\lambda\delta_N(x_t)
    +\varepsilon_t^Y
    -\lambda\varepsilon_t^Z.
\end{aligned}
\end{equation}
The corrected outcome therefore has conditional mean
\(f(x_t)+\lambda\delta_N(x_t)\) and conditional variance \(\eta_\lambda^2\).
This spatial mean shift can persist as online observations accumulate and alter subsequent GP-UCB queries. The next subsection describes how the correction coefficient is estimated and updated using the online observations.

\subsection{Correction Coefficient}
\label{subsec:correction-coefficient}

The oracle coefficient in \eqref{eq:def-oracle-coefficient} is generally unknown. When \(\widehat g_N=g\), the auxiliary residual satisfies \(\widetilde Z_t=\varepsilon_t^Z\). Using the conditional centering in
\eqref{eq:online-conditional-centering}, we have
\begin{equation}
\label{eq:coefficient-moment-identities}
\begin{aligned}
    \E\!\left[
        Y_t\widetilde Z_t
        \mid \F_{t-1}
    \right]
    &=
    \sigma_{YZ},\\
    \E\!\left[
        \widetilde Z_t^2
        \mid \F_{t-1}
    \right]
    &=
    \sigma_Z^2.
\end{aligned}
\end{equation}
Replacing the expectations in
\eqref{eq:coefficient-moment-identities} with averages over the observed pairs gives the online estimator
\begin{equation}
\label{eq:def-online-coefficient}
    \widehat\lambda_t
    :=
    \frac{
        \sum_{s=1}^t Y_s\widetilde Z_s
    }{
        \sum_{s=1}^t \widetilde Z_s^2
    }.
\end{equation}
Since the noise variances and covariance do not vary with \(x\), the estimator combines observations from different locations without requiring repeated queries. It uses only the observed pairs and does not require \(\sigma_{YZ}\) or \(\sigma_Z^2\) as inputs.  When
\(\widehat g_N=g\), it estimates \(\lambda^\star\).  With estimated centering, \(\widetilde Z_t\) also contains \(-\delta_N(x_t)\), which may shift \(\widehat\lambda_t\) away from \(\lambda^\star\).  We quantify this effect in the analysis.

At round \(t\), the coefficient used to select \(x_t\) is
\(\widehat\lambda_{t-1}\), which is computed before observing
\((Y_t,Z_t)\). We apply this coefficient to the full history and define
\begin{equation}
\label{eq:def-recorrected-history}
    \widetilde Y_s^{(t)}
    :=
    Y_s-\widehat\lambda_{t-1}\widetilde Z_s,
    \qquad
    s=1,\ldots,t-1.
\end{equation}
Thus, all past observations are corrected using the same coefficient before the round-\(t\) GP update.  After observing \((Y_t,Z_t)\), the estimator is updated to \(\widehat\lambda_t\) for the next round. The next subsection uses the corrected history to construct the GP-UCB acquisition rule.

\subsection{GP-UCB Update}
\label{subsec:gp-ucb-update}

At round \(t\), we collect the corrected responses in
\[
    \widetilde{\mathbf Y}_{t-1}^{(t)}
    :=
    \left(
        \widetilde Y_1^{(t)},
        \ldots,
        \widetilde Y_{t-1}^{(t)}
    \right)^\top.
\]
Let
\[
    \mathbf k_{t-1}(x)
    :=
    \left(
        k(x_1,x),
        \ldots,
        k(x_{t-1},x)
    \right)^\top,
    \qquad
    \mathbf K_{t-1}
    :=
    \left[k(x_i,x_j)\right]_{i,j=1}^{t-1}.
\]
We use a zero-mean GP with kernel \(k\).  At each round, its scalar noise variance \(\widehat\eta_{t-1}^2\) is estimated from the corrected history by maximizing the marginal likelihood with \(k\) held fixed. Accordingly, Algorithm~\ref{alg:ppi-gp-ucb} does not require \(\Sigma\) as an input.  The resulting posterior mean and variance are
\begin{equation}
\label{eq:corrected-gp-posterior}
\begin{aligned}
    \mu_{t-1}(x)
    &:=
    \mathbf k_{t-1}(x)^\top
    \left(
        \mathbf K_{t-1}
        +
        \widehat\eta_{t-1}^2\mathbf I_{t-1}
    \right)^{-1}
    \widetilde{\mathbf Y}_{t-1}^{(t)},\\
    s_{t-1}^2(x)
    &:=
    k(x,x)
    -
    \mathbf k_{t-1}(x)^\top
    \left(
        \mathbf K_{t-1}
        +
        \widehat\eta_{t-1}^2\mathbf I_{t-1}
    \right)^{-1}
    \mathbf k_{t-1}(x).
\end{aligned}
\end{equation}
Given an exploration parameter \(\beta_t>0\), the next query is selected by
\begin{equation}
\label{eq:ppi-gp-ucb-acquisition}
    x_t
    \in
    \arg\max_{x\in\X}
    \left\{
        \mu_{t-1}(x)
        +
        \sqrt{\beta_t}\,s_{t-1}(x)
    \right\}.
\end{equation}
The complete AI-GPB procedure is summarized in Algorithm~\ref{alg:ppi-gp-ucb}. It uses the standard GP-UCB acquisition rule with a posterior constructed from corrected observations. The next section analyzes a stabilized version and shows how offline mean estimation and online coefficient learning enter the cumulative-regret guarantee.

\begin{algorithm}[t]
\caption{Auxiliary-Informed Gaussian Process Bandit (AI-GPB)}
\label{alg:ppi-gp-ucb}
\textbf{Input:} search space \(\X\), horizon \(T\), initial size \(n_0\), offline estimate \(\widehat g_N\), kernel \(k\), and exploration parameters
\(\{\beta_t\}_{t=1}^T\). \\
\textbf{Initialize:} select \(n_0\) initial points and observe their paired
outcomes. Compute
\(\widetilde Z_i=Z_i-\widehat g_N(x_i)\) for \(i=1,\ldots,n_0\).

\begin{algorithmic}[1]
    \FOR{$t=n_0,\ldots,T-1$}
        \STATE \textbf{Correction update}
        \STATE Compute \(\widehat\lambda_t\) using
        \eqref{eq:def-online-coefficient}
        \STATE Correct the observation history by setting \(\widetilde Y_s^{(t+1)}
        \leftarrow
        Y_s-\widehat\lambda_t\widetilde Z_s\)
        for \(s=1,\ldots,t\)
        \STATE \textbf{GP update}
        \STATE Estimate \(\widehat\eta_t^2\) by maximizing the marginal
        likelihood of the corrected history
        \STATE Compute \(\mu_t(\cdot)\) and \(s_t^2(\cdot)\) using
        \eqref{eq:corrected-gp-posterior}
        \STATE \textbf{Acquisition and query}
        \STATE
        \(x_{t+1}
        \leftarrow
        \arg\max_{x\in\X}
        \left\{
            \mu_t(x)+\sqrt{\beta_{t+1}}\,s_t(x)
        \right\}\)
        \STATE Observe the paired outcomes
        \((Y_{t+1},Z_{t+1})\)
        \STATE Compute
        \(\widetilde Z_{t+1}
        \leftarrow
        Z_{t+1}-\widehat g_N(x_{t+1})\)
    \ENDFOR
\end{algorithmic}
\end{algorithm}
\section{Theoretical Analysis}
\label{sec:performance-guarantees}

We first control the online coefficient estimator. We then establish GP confidence bounds and cumulative-regret guarantees for a stabilized version of AI-GPB.

\subsection{Coefficient Estimation}
\label{subsec:coefficient-estimation}

We begin by identifying the coefficient learned from the online paired observations. Conditional on the offline dataset, \(x_t\) and
\(\delta_N(x_t)\) are determined before \((Y_t,Z_t)\) is observed.  Using
\(\widetilde Z_t=\varepsilon_t^Z-\delta_N(x_t)\), we obtain
\begin{equation}
\label{eq:coefficient-conditional-moments}
\begin{aligned}
    \E\!\left[
        Y_t\widetilde Z_t
        \mid\F_{t-1}
    \right]
    &=
    \sigma_{YZ}
    -
    f(x_t)\delta_N(x_t),\\
    \E\!\left[
        \widetilde Z_t^2
        \mid\F_{t-1}
    \right]
    &=
    \sigma_Z^2
    +
    \delta_N^2(x_t).
\end{aligned}
\end{equation}
Taking the ratio of the averages of these conditional moments over the first \(t\) rounds gives the population counterpart of \(\widehat\lambda_t\):
\begin{equation}
\label{eq:def-population-coefficient}
    \lambda_{N,t}
    :=
    \frac{
        \sigma_{YZ}
        -
        \frac{1}{t}
        \sum_{s=1}^t
        f(x_s)\delta_N(x_s)
    }{
        \sigma_Z^2
        +
        \frac{1}{t}
        \sum_{s=1}^t
        \delta_N^2(x_s)
    }.
\end{equation}
The coefficient \(\lambda_{N,t}\) differs from \(\lambda^\star\) only through the offline errors \(\delta_N(x_s)\) at the queried locations. It therefore equals \(\lambda^\star\) when \(\widehat g_N=g\).

Subtracting the conditional means in \eqref{eq:coefficient-conditional-moments} gives martingale differences, despite the adaptive choice of query locations.  This yields the following concentration result.

\begin{theorem}[Coefficient concentration]
\label{theorem:online-coefficient-concentration}
Suppose Assumptions~\ref{assumption:kernel-objective} and
\ref{assumption:paired-fluctuations} hold.  Condition on an offline dataset
satisfying
\[
    \|\delta_N\|_\infty\leq e_N.
\]
For any horizon \(T\geq2\) and \(\alpha\in(0,1)\), there exists a constant
\(C>0\), depending only on \(B\), \(e_N\), and \(\Sigma\), such that with probability at least \(1-\alpha\) over the online observations,
\begin{equation}
\label{eq:online-coefficient-concentration}
    \left|
        \widehat\lambda_t-\lambda_{N,t}
    \right|
    \leq
    a_t
    :=
    C
    \sqrt{
        \frac{\log(T/\alpha)}{t}
    },
    \qquad
    C\log(T/\alpha)\leq t\leq T.
\end{equation}
\end{theorem}

The proof uses martingale concentration for the two empirical moments and is
given in the appendix.  Assumption~\ref{assumption:kernel-objective} implies
\(|f(x)|\leq B\) for every \(x\in\X\), which allows us to compare
\(\lambda_{N,t}\) with the oracle coefficient.

\begin{corollary}[Oracle comparison]
\label{corollary:online-coefficient-oracle-gap}
Under the conditions of
Theorem~\ref{theorem:online-coefficient-concentration}, define
\begin{equation}
\label{eq:def-offline-coefficient-error}
    q_N
    :=
    \frac{Be_N}{\sigma_Z^2}
    +
    \frac{|\sigma_{YZ}|e_N^2}{\sigma_Z^4}.
\end{equation}
On the same event, for every \(t\) in the range of
\eqref{eq:online-coefficient-concentration},
\begin{equation}
\label{eq:online-coefficient-oracle-gap}
\begin{aligned}
    \left|
        \widehat\lambda_t-\lambda^\star
    \right|
    &\leq
    a_t+q_N,\\
    0
    \leq
    \eta_{\widehat\lambda_t}^2-\eta_\star^2
    &=
    \sigma_Z^2
    \left(
        \widehat\lambda_t-\lambda^\star
    \right)^2
    \leq
    \sigma_Z^2(a_t+q_N)^2.
\end{aligned}
\end{equation}
\end{corollary}

The variance gap in \eqref{eq:online-coefficient-oracle-gap} is \(O(\log(T/\alpha)/t)\) with exact centering \((\widehat g_N=g)\), for which \(q_N=0\). With estimated centering, \(q_N=O(e_N)\), and the upper bound approaches \(\sigma_Z^2q_N^2=O(e_N^2)\). The coefficient error analysis also extends to reward shifts and positive rescaling, as shown in Lemma~\ref{lemma:response-transformation} in Appendix~\ref{appendix:coefficient-concentration}.

Offline estimation error also shifts the conditional mean:
\begin{equation}
\label{eq:corrected-conditional-mean}
    \E\!\left[
        \widetilde Y_{t+1}(\widehat\lambda_t)
        \mid\F_t
    \right]
    =
    f(x_{t+1})
    +
    \widehat\lambda_t\delta_N(x_{t+1}).
\end{equation}
The second term is \(\F_t\)-measurable, is of order \(e_N\), and need not vanish as \(t\) grows. The next subsection incorporates this mean shift into the GP confidence bound.

\subsection{Confidence Bounds}
\label{subsec:confidence-bounds}

To obtain a direct confidence bound, we analyze a stabilized version of AI-GPB. Let \(m<T\) be a round covered by Theorem~\ref{theorem:online-coefficient-concentration}. At the end of round \(m\), we fix the coefficient at \(\widehat\lambda_m\) and use it to correct the full observation history. The full procedure is given in the appendix.

For each \(t>m\), we form the corrected history by setting
\begin{equation}
\label{eq:def-stabilized-observation}
\begin{aligned}
    \widetilde Y_s^{(m)}
    &:=
    Y_s-\widehat\lambda_m\widetilde Z_s\\
    &=
    h_m(x_s)+\widetilde\varepsilon_s^{(m)},
    \qquad
    s=1,\ldots,t-1.
\end{aligned}
\end{equation}
where
\begin{equation}
\label{eq:def-stabilized-target-noise}
    h_m(x)
    :=
    f(x)+\widehat\lambda_m\delta_N(x),
    \qquad
    \widetilde\varepsilon_s^{(m)}
    :=
    \varepsilon_s^Y-\widehat\lambda_m\varepsilon_s^Z.
\end{equation}
Conditional on \(\F_m\), the function \(h_m\) is fixed.  For every \(s>m\), the corrected noise satisfies
\begin{equation}
\label{eq:stabilized-noise-moments}
    \E\!\left[
        \widetilde\varepsilon_s^{(m)}
        \mid\F_{s-1}
    \right]
    =
    0,
    \qquad
    \Var\!\left(
        \widetilde\varepsilon_s^{(m)}
        \mid\F_{s-1}
    \right)
    =
    \eta_{\widehat\lambda_m}^2.
\end{equation}
The first \(m\) observations are used to estimate \(\widehat\lambda_m\). We account for this dependence through a confidence bound uniform over the coefficient range. The proof in Appendix~\ref{appendix:confidence-corrected-history} controls the oracle residual and auxiliary noise separately.

On the event in Corollary~\ref{corollary:online-coefficient-oracle-gap},
\begin{equation}
\label{eq:stabilized-noise-envelope}
\begin{aligned}
    \eta_{\widehat\lambda_m}^2
    &=
    \eta_\star^2
    +
    \sigma_Z^2
    \left(
        \widehat\lambda_m-\lambda^\star
    \right)^2\\
    &\leq
    \overline\eta_m^2
    :=
    \eta_\star^2
    +
    \sigma_Z^2(a_m+q_N)^2.
\end{aligned}
\end{equation}
The confidence analysis uses \(\overline\eta_m^2\) as a fixed GP noise parameter. In implementation, Algorithm~\ref{alg:ppi-gp-ucb} instead estimates the GP noise variance from the corrected observations.

To apply the GP confidence result to \(h_m\), we additionally assume that, on the offline event in Assumption~\ref{assumption:offline-estimation},
\begin{equation}
\label{eq:offline-rkhs-error-bound}
    \delta_N\in\mathcal H_k,
    \qquad
    \|\delta_N\|_{\mathcal H_k}\leq D,
\end{equation}
for some constant \(D\geq 0\) that does not depend on \(N\) or \(T\). This RKHS condition is separate from the uniform error bound in \eqref{eq:offline-uniform-error}.  Together with Assumption~\ref{assumption:kernel-objective}, it gives
\begin{equation}
\label{eq:stabilized-target-rkhs-bound}
    \|h_m\|_{\mathcal H_k}
    \leq
    B+|\widehat\lambda_m|D.
\end{equation}

For \(n\geq 1\), define the maximum information gain under noise variance \(\overline\eta_m^2\) by
\begin{equation}
\label{eq:def-maximum-information-gain}
    \gamma_n\!\left(k,\overline\eta_m^2\right)
    :=
    \max_{A\in\X^n}
    \frac{1}{2}
    \log\det\!\left(
        \mathbf I_n
        +
        \overline\eta_m^{-2}\mathbf K_A
    \right),
\end{equation}
where \(A=(x_1,\ldots,x_n)\) and \(\mathbf K_A=[k(x_i,x_j)]_{i,j=1}^n\).  For \(t>m\), let \(\mu_{t-1}\) and \(s_{t-1}\) denote the GP posterior mean and standard deviation computed from the full corrected history
\(\{(x_s,\widetilde Y_s^{(m)})\}_{s=1}^{t-1}\), using
\(\overline\eta_m^2\) as the noise parameter. Given \(\alpha\in(0,1)\), set
\begin{equation}
\label{eq:def-stabilized-confidence-parameter}
    \sqrt{\beta_t}
    :=
    B+|\widehat\lambda_m|D
    +
    2
    \sqrt{
        \gamma_{t-1}\!\left(k,\overline\eta_m^2\right)
        +
        1
        +
        \log\!\left(\frac{4}{\alpha}\right)
    }.
\end{equation}

\begin{theorem}[Confidence bound]
\label{theorem:confidence-online-correction}
Suppose Assumptions~\ref{assumption:kernel-objective} and \ref{assumption:paired-fluctuations} hold.  Condition on an offline dataset satisfying \eqref{eq:offline-uniform-error} and \eqref{eq:offline-rkhs-error-bound}. Then, with probability at least \(1-\alpha\) over the online observations,
\begin{equation}
\label{eq:confidence-stabilized-target}
    \left|
        h_m(x)-\mu_{t-1}(x)
    \right|
    \leq
    \sqrt{\beta_t}\,s_{t-1}(x)
\end{equation}
holds simultaneously for every \(t=m+1,\ldots,T\) and \(x\in\X\).
\end{theorem}

The coefficient error enters \eqref{eq:stabilized-noise-envelope} quadratically, while the offline estimation error shifts the target function by
\begin{equation}
\label{eq:stabilized-target-uniform-gap}
    \|h_m-f\|_\infty
    \leq
    |\widehat\lambda_m|e_N.
\end{equation}
The next subsection transfers the confidence bound for \(h_m\) to the original objective \(f\) and derives the resulting cumulative regret bound.

\subsection{Cumulative Regret}
\label{subsec:cumulative-regret}

We now transfer the confidence bound for \(h_m\) to regret for the original objective \(f\).  For \(t>m\), the stabilized procedure uses the GP-UCB acquisition rule with the posterior and confidence parameter defined in the previous subsection. On the event in Theorem~\ref{theorem:confidence-online-correction},
\begin{equation}
\label{eq:pointwise-regret-transfer}
\begin{aligned}
    r_t
    &=
    f(x^\star)-f(x_t)\\
    &\leq
    h_m(x^\star)-h_m(x_t)
    +
    2|\widehat\lambda_m|e_N\\
    &\leq
    2\sqrt{\beta_t}\,s_{t-1}(x_t)
    +
    2|\widehat\lambda_m|e_N.
\end{aligned}
\end{equation}
The first inequality follows from \eqref{eq:stabilized-target-uniform-gap}.  The second follows from the confidence bound and the GP-UCB selection rule. To sum \eqref{eq:pointwise-regret-transfer}, we first control the posterior standard deviations at the queried points.

\begin{lemma}[Cumulative posterior width]
\label{lemma:cumulative-posterior-width}
For the posterior standard deviations in Subsection~\ref{subsec:confidence-bounds}, there exists a constant \(C>0\) such that
\begin{equation}
\label{eq:cumulative-posterior-width}
    \sum_{t=m+1}^T s_{t-1}(x_t)
    \leq
    C
    \left\{
        \overline\eta_m
        \sqrt{
            (T-m)\gamma_T\!\left(k,\overline\eta_m^2\right)
        }
        +
        \gamma_T\!\left(k,\overline\eta_m^2\right)
    \right\}.
\end{equation}
\end{lemma}

The proof separates the post-\(m\) rounds according to whether \(s_{t-1}(x_t)\) is below or above \(\overline\eta_m\), and then applies the information-gain identity. The explicit constants and proof are given in the appendix.

\begin{theorem}[Cumulative regret]
\label{theorem:cumulative-regret}
Under the conditions of Theorem~\ref{theorem:confidence-online-correction}, there exists a constant \(C>0\) such that, with probability at least \(1-\alpha\) over the online observations,
\begin{equation}
\label{eq:general-cumulative-regret}
\begin{aligned}
    R_T
    \leq{}&
    2Bm\\
    &+
    C\sqrt{\beta_T}
    \left\{
        \overline\eta_m
        \sqrt{
            (T-m)
            \gamma_T\!\left(k,\overline\eta_m^2\right)
        }
        +
        \gamma_T\!\left(k,\overline\eta_m^2\right)
    \right\}\\
    &+
    2(T-m)|\widehat\lambda_m|e_N.
\end{aligned}
\end{equation}
\end{theorem}

Together with Assumption~\ref{assumption:offline-estimation}, the bound holds jointly over the offline and online observations with probability at least \(1-\alpha_N-\alpha\). The three terms in \eqref{eq:general-cumulative-regret} account for the first \(m\) rounds, GP-UCB learning under the corrected noise level, and the transfer from \(h_m\) to \(f\), respectively. When \(\widehat g_N=g\), the transfer term vanishes and
\[
    \overline\eta_m^2-\eta_\star^2
    =
    O\!\left(
        \frac{\log(4T/\alpha)}{m}
    \right).
\]
Thus, the corrected noise bound approaches the oracle variance as the coefficient estimate improves.

The next subsection specializes \eqref{eq:general-cumulative-regret} to Matérn kernels and relates the required offline accuracy to the offline sample size.

\subsection{Kernel-Specific Rates}
\label{subsec:kernel-specific-rates}

Theorem~\ref{theorem:cumulative-regret} applies to any kernel for which the maximum information gain can be controlled.  We now specialize the result to Matérn kernels, which are also used in our numerical experiments. Let \(\X\subset\mathbb R^d\), and let \(k\) be a stationary Matérn kernel with smoothness \(\nu>d/2\) and fixed hyperparameters. Appendix~\ref{appendix:kernel-specific-rates} establishes
\begin{equation}
\label{eq:matern-information-gain}
    \gamma_T(k,\eta^2)
    =
    \widetilde O\!\left(
        \left(
            \frac{T}{\eta^2}
        \right)^{\frac{d}{2\nu+d}}
    \right).
\end{equation}
Here, \(\widetilde O\) suppresses logarithmic factors, while the dependence on the noise variance is shown explicitly. Combining this bound with \eqref{eq:general-cumulative-regret} and the identity \(\eta_\star^2=\sigma_Y^2(1-\rho^2)\) gives the following correlation-dependent regret rate.

\begin{corollary}[Matérn regret rate]
\label{corollary:matern-regret-rate}
Suppose the conditions of Theorem~\ref{theorem:cumulative-regret} hold, and fix \(\rho\in(-1,1)\) independently of \(T\). Let
\begin{equation}
\label{eq:matern-rate-conditions}
    m
    =
    \left\lceil\sqrt T\right\rceil,
    \qquad
    e_N
    =
    \widetilde O\!\left(
        T^{-\frac{\nu}{2\nu+d}}
    \right).
\end{equation}
Then, with probability at least \(1-\alpha_N-\alpha\),
\begin{equation}
\label{eq:matern-cumulative-regret-rate}
    R_T
    =
    \widetilde O\!\left(
        (1-\rho^2)^{\frac{2\nu-d}{2(2\nu+d)}}
        T^{\frac{2\nu+3d}{2(2\nu+d)}}
        +
        T^{\frac{\nu+d}{2\nu+d}}
    \right).
\end{equation}
\end{corollary}
The conditions in \eqref{eq:matern-rate-conditions} give
\[
    \overline\eta_m^2
    =
    \eta_\star^2
    +
    \widetilde O(T^{-1/2}).
\]
This choice of \(m\) preserves the leading noise-dependent term. Regret from the first \(m\) rounds and offline estimation error enter only the lower-order term in \eqref{eq:matern-cumulative-regret-rate}.

For a Matérn-\(5/2\) kernel in two dimensions,
\eqref{eq:matern-cumulative-regret-rate} becomes
\begin{equation}
\label{eq:matern-five-halves-two-dimensional-rate}
    R_T
    =
    \widetilde O\!\left(
        (1-\rho^2)^{3/14}T^{11/14}
        +
        T^{9/14}
    \right).
\end{equation}
Vanilla GP-UCB uses \(\sigma_Y^2\) in the same Matérn information-gain bound, while the leading term in \eqref{eq:matern-five-halves-two-dimensional-rate} uses \(\eta_\star^2=\sigma_Y^2(1-\rho^2)\).  Consequently, the polynomial noise dependence of the leading upper-bound term is reduced by the factor \((1-\rho^2)^{3/14}\).

Under the Matérn kernel ridge regression conditions stated in Appendix~\ref{appendix:kernel-specific-rates}, the estimator constructed there satisfies
\eqref{eq:offline-rkhs-error-bound} and achieves
\begin{equation}
\label{eq:matern-offline-supnorm-rate}
    \|\widehat g_N-g\|_\infty
    =
    \widetilde O_{\mathbb P}\!\left(
        N^{-\frac{\nu}{2\nu+d}}
    \right).
\end{equation}
At a fixed offline confidence level, \(e_N\) can therefore be chosen at the same rate.  It follows that \(N=\widetilde\Omega(T)\) is sufficient for
\eqref{eq:matern-rate-conditions}.  

Finally, \(S_T\leq R_T/T\) gives
\begin{equation}
\label{eq:matern-simple-regret-rate}
    S_T
    =
    \widetilde O\!\left(
        (1-\rho^2)^{\frac{2\nu-d}{2(2\nu+d)}}
        T^{-\frac{2\nu-d}{2(2\nu+d)}}
        +
        T^{-\frac{\nu}{2\nu+d}}
    \right).
\end{equation}
Together, these bounds show how the correlation, kernel smoothness, and offline estimation accuracy enter the performance guarantee.  The next section examines these effects numerically.

\section{Experiments}
\label{sec:experiments}

A central motivation for paired correction is to use within-pair dependence to reduce observation noise and improve sequential decisions, even when the auxiliary mean is estimated offline. We examine this behavior on continuous benchmark functions with different optimization landscapes.  Unlike independent finite-arm problems, each observation changes GP uncertainty beyond the queried point.  We vary \(\rho\) while holding \(f\), \(g\), and both noise variances fixed. The shape of \(g\) is chosen independently of \(f\), so the auxiliary source is not designed to predict the objective.

To evaluate the gains from auxiliary observations with known and estimated means, we compare Vanilla GP-UCB, a known-\(g\) reference, and AI-GPB. All three use the GP-UCB acquisition rule. The reference uses the true auxiliary mean, whereas AI-GPB constructs \(\widehat g_N\) from offline auxiliary observations. We report both cumulative and simple regret across benchmarks. Appendix~\ref{appendix:experimental-setup} gives the experimental protocol, and Appendix~\ref{appendix:additional-experiments} reports coefficient-learning diagnostics, offline-budget comparisons, and exploration sensitivity.

\subsection{Cumulative Regret}
\label{subsec:experimental-cumulative-regret}
We evaluate AI-GPB across benchmark functions with different optimization landscapes. Figure~\ref{fig:cumulative-regret} shows lower cumulative regret than Vanilla GP-UCB on all four tasks, with larger reductions as \(\rho\) increases. The gap widens with \(T\), indicating that the improvement is not confined to the initial rounds.

\begin{figure}[t]
    \centering
    \includegraphics[width=\textwidth]{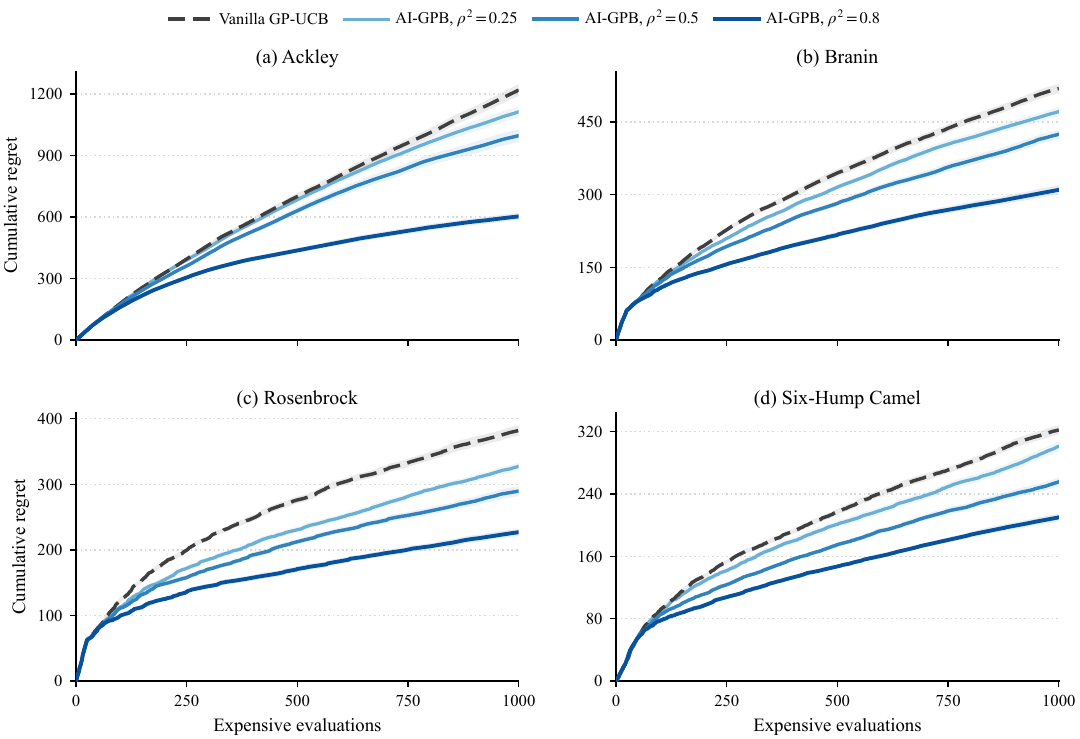}
    \caption{\textbf{Cumulative regret over the online horizon.}
    Curves report means over 20 replications, and shading
    denotes one standard error of the mean. All methods share the same initialization within each replication.}
    \label{fig:cumulative-regret}
\end{figure}

\begin{table}[!htbp]
    \centering
    \caption{\textbf{Cumulative regret at \(T=1000\).}
    Results are averaged over 20 replications. AI-GPB entries also report percentage reductions relative to Vanilla GP-UCB.}
    \label{tab:cumulative-regret-reduction}
    \begin{tabular*}{\textwidth}{
        @{\extracolsep{\fill}}
        lccccc
        @{}
    }
        \toprule
        &
        &
        \multicolumn{4}{c}{Mean cumulative regret \(R_{1000}\)} \\
        \cmidrule(lr){3-6}
        Method
        & \(\rho^2\)
        & Ackley
        & Branin
        & Rosenbrock
        & Six-Hump Camel \\
        \midrule
        Vanilla GP-UCB
        & --
        & \(1219.4\)
        & \(518.9\)
        & \(382.2\)
        & \(322.2\) \\
        \addlinespace[2pt]
        AI-GPB
        & \(0.25\)
        & \(1113.5\;(8.7\%)\)
        & \(471.3\;(9.2\%)\)
        & \(327.4\;(14.3\%)\)
        & \(301.0\;(6.6\%)\) \\
        & \(0.5\)
        & \(996.9\;(18.2\%)\)
        & \(424.5\;(18.2\%)\)
        & \(289.8\;(24.2\%)\)
        & \(255.7\;(20.6\%)\) \\
        & \(0.8\)
        & \(603.0\;(50.6\%)\)
        & \(309.9\;(40.3\%)\)
        & \(227.1\;(40.6\%)\)
        & \(210.1\;(34.8\%)\) \\
        \bottomrule
    \end{tabular*}
\end{table}

At \(T=1000\), AI-GPB reduces cumulative regret relative to Vanilla GP-UCB across all tasks and values of \(\rho\), as reported in Table~\ref{tab:cumulative-regret-reduction}.  At \(\rho^2=0.8\), the reductions range from \(34.8\%\) to \(50.6\%\). Writing \(c=1-\rho^2\), the noise dependence of Matérn regret upper bounds provides two reference values: \(1-c^{3/14}\) from \eqref{eq:matern-five-halves-two-dimensional-rate} and \(1-c^{5/14}\) from sharper Matérn regret bounds \citep{iwazaki2025improved}. Numerically, the observed reductions exceed the former and are generally close to the latter.

\subsection{Simple Regret}
\label{subsec:experimental-simple-regret}

Cumulative regret evaluates the full sampling path, whereas the simple regret in \eqref{eq:def-simple-regret} measures the best point found within a fixed budget of expensive evaluations \citep{wang2016optimization}.  The simple-regret rate in \eqref{eq:matern-simple-regret-rate} follows directly by applying
\(S_T\leq R_T/T\) to the cumulative-regret bound.

AI-GPB uses auxiliary residuals to reduce observation noise, making each expensive evaluation more informative. Through the GP kernel, the resulting gain in posterior precision extends to related locations and can influence subsequent GP-UCB queries.  We examine this effect over \(300\) iterations on six standard BO benchmarks with dimensions ranging from \(2\) to \(10\).

Figure~\ref{fig:simple-regret} reports the simple regret trajectories across benchmarks.  At \(\rho^2=0.8\), the figure also includes a reference using the true \(g\) and the population-optimal coefficient \(\lambda^\star\). Performance generally improves as \(\rho\) increases. After \(300\) adaptive evaluations, AI-GPB with \(\rho^2=0.8\) achieves lower mean log simple regret than Vanilla GP-UCB on every task, corresponding to reductions of \(22.7\%\)--\(72.4\%\) in geometric mean simple regret. 

\begin{figure}[t]
    \centering
    \includegraphics[width=\textwidth]{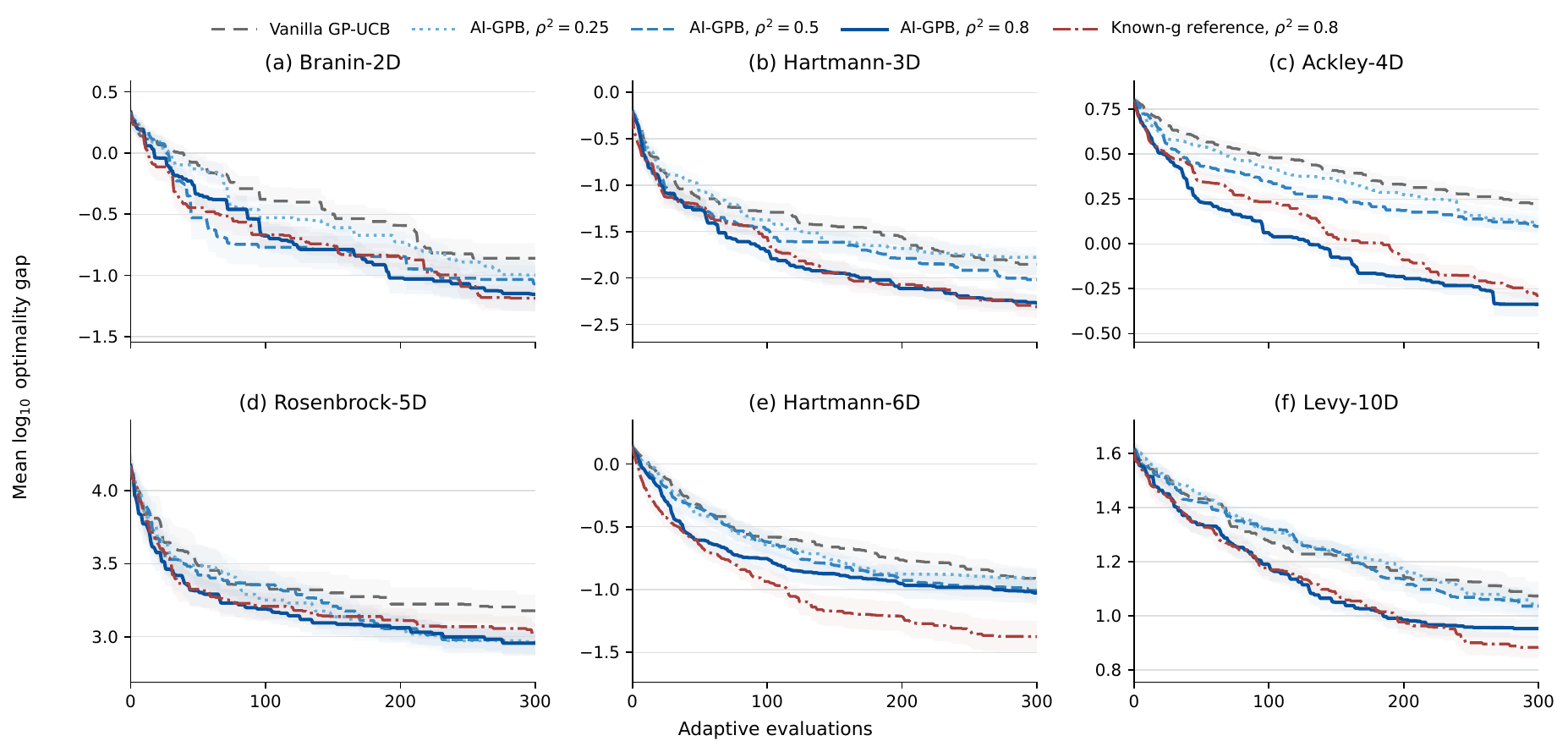}
    \caption{\textbf{Simple regret over the adaptive horizon.}
    Curves report mean \(\log_{10}S_t\) over 20 replications, and shading denotes one standard error of the mean. All methods share the same initialization within each replication.}
    \label{fig:simple-regret}
\end{figure}

At the terminal horizon, the geometric mean simple regret of AI-GPB is within \(17.3\%\) of the reference on five of the six tasks, with a larger gap on Hartmann-6D. Because the methods follow different adaptive query paths, the reference need not dominate on every task. Consistent with Corollary~\ref{corollary:online-coefficient-oracle-gap}, online coefficient estimation can recover much of the variance reduction achieved by \(\lambda^\star\) when \(\widehat g_N\) is accurate.  Together, these results illustrate how reducing observation variance can improve the search path and help AI-GPB find better solutions with fewer expensive evaluations.

\section{Conclusion}
\label{sec:conclusion}
In this work, we formulate a framework for continuous GP bandits with offline auxiliary data and paired online observations. We develop AI-GPB, an algorithm that reduces observation noise through correlation between paired errors without requiring the auxiliary mean to predict the objective. Our analysis establishes kernel-dependent regret guarantees with a stabilized correction coefficient, accounting for offline mean estimation and online coefficient learning. We empirically demonstrate improvements in both cumulative and simple regret across benchmarks with varied landscapes and dimensions.

While we use GP-UCB for query selection, paired correction enters the GP posterior update and can be combined with other acquisition rules. On the modeling side, a useful extension is to allow within-pair dependence to vary across the search space. We view both extensions as promising directions for future studies.

\bibliographystyle{plain}
\bibliography{main}

\appendix

\section{Proofs for Theoretical Analysis}
\label{appendix:theory-proofs}

\subsection{Proof Setup}
\label{appendix:proof-setup}

Throughout this appendix, we condition on the offline event in \eqref{eq:offline-uniform-error}.  The RKHS condition in \eqref{eq:offline-rkhs-error-bound} is used only for the confidence and regret results.

\paragraph{Coefficient range.}
Fix a horizon \(T\) and a confidence level \(\alpha\in(0,1)\).  Choose \(m<T\) in the concentration range of Theorem~\ref{theorem:online-coefficient-concentration} when applied with failure probability \(\alpha/2\), and let \(a_m\) denote the corresponding radius. Define
\begin{equation}
\label{eq:appendix-coefficient-range}
    d_m
    :=
    a_m+q_N,
    \qquad
    \mathcal I_m
    :=
    [\lambda^\star-d_m,\lambda^\star+d_m].
\end{equation}
Corollary~\ref{corollary:online-coefficient-oracle-gap} gives
\(\widehat\lambda_m\in\mathcal I_m\) with probability at least \(1-\alpha/2\). For every \(\lambda\in\mathcal I_m\),
\begin{equation}
\label{eq:appendix-common-noise-envelope}
    \eta_\lambda^2
    =
    \eta_\star^2
    +
    \sigma_Z^2(\lambda-\lambda^\star)^2
    \leq
    \overline\eta_m^2
    :=
    \eta_\star^2+\sigma_Z^2d_m^2.
\end{equation}
The remaining failure probability \(\alpha/2\) is assigned to the GP confidence event.

\paragraph{Fixed correction.}
For a fixed coefficient \(\lambda\), define
\begin{equation}
\label{eq:appendix-fixed-correction-model}
    h_\lambda(x)
    :=
    f(x)+\lambda\delta_N(x),
    \qquad
    \varepsilon_s^\lambda
    :=
    \varepsilon_s^Y-\lambda\varepsilon_s^Z.
\end{equation}
Then
\begin{equation}
\label{eq:appendix-fixed-corrected-observation}
    \widetilde Y_s(\lambda)
    =
    h_\lambda(x_s)+\varepsilon_s^\lambda.
\end{equation}
Setting \(\lambda=\widehat\lambda_m\) recovers
\eqref{eq:def-stabilized-observation}--%
\eqref{eq:def-stabilized-target-noise}.
The first \(m\) queries may be selected by any policy for which
\(x_t\) is \(\F_{t-1}\)-measurable. After round \(m\), the stabilized policy fixes \(\widehat\lambda_m\). For \(t>m\), it forms the GP posterior from the full corrected history with noise parameter \(\overline\eta_m^2\) and selects \(x_t\) using the GP-UCB acquisition rule. The first \(m\) observations therefore remain in the GP update. AI-GPB (Algorithm~\ref{alg:ppi-gp-ucb}) continues to update the coefficient and noise parameter online.

\subsection{Coefficient Concentration}
\label{appendix:coefficient-concentration}

Throughout this subsection, we condition on an offline dataset satisfying
\[
    \|\delta_N\|_\infty\leq e_N.
\]
For \(t\geq1\), define the empirical deviations
\begin{equation}
\label{eq:appendix-centered-empirical-moments}
\begin{aligned}
    \Delta_{1,t}
    &:=
    \frac{1}{t}
    \sum_{s=1}^t
    \left\{
        Y_s\widetilde Z_s
        -
        \E\!\left[
            Y_s\widetilde Z_s
            \mid\F_{s-1}
        \right]
    \right\},\\
    \Delta_{2,t}
    &:=
    \frac{1}{t}
    \sum_{s=1}^t
    \left\{
        \widetilde Z_s^2
        -
        \E\!\left[
            \widetilde Z_s^2
            \mid\F_{s-1}
        \right]
    \right\}.
\end{aligned}
\end{equation}

\begin{lemma}[Online moment concentration]
\label{lemma:online-moment-concentration}
Suppose Assumptions~\ref{assumption:kernel-objective} and
\ref{assumption:paired-fluctuations} hold.  For any \(T\geq1\) and \(\alpha\in(0,1)\), there exists a constant \(C_0>0\), depending only on
\(B\), \(e_N\), and \(\Sigma\), such that, with probability at least \(1-\alpha\),
\begin{equation}
\label{eq:appendix-online-moment-concentration}
\begin{gathered}
    |\Delta_{1,t}|
    \leq b_t,
    \qquad
    |\Delta_{2,t}|
    \leq b_t,\\
    b_t
    :=
    C_0
    \left\{
        \sqrt{
            \frac{\log(4T/\alpha)}{t}
        }
        +
        \frac{\log(4T/\alpha)}{t}
    \right\}
\end{gathered}
\end{equation}
holds simultaneously for every \(t=1,\ldots,T\).
\end{lemma}

\begin{proof}
Because \(x_s\) and \(\delta_N(x_s)\) are
\(\F_{s-1}\)-measurable, the summands in
\eqref{eq:appendix-centered-empirical-moments} satisfy
\begin{equation}
\label{eq:appendix-centered-moment-increments}
\begin{aligned}
    Y_s\widetilde Z_s
    -
    \E\!\left[
        Y_s\widetilde Z_s
        \mid\F_{s-1}
    \right]
    &=
    f(x_s)\varepsilon_s^Z
    -
    \delta_N(x_s)\varepsilon_s^Y
    +
    \varepsilon_s^Y\varepsilon_s^Z
    -
    \sigma_{YZ},\\
    \widetilde Z_s^2
    -
    \E\!\left[
        \widetilde Z_s^2
        \mid\F_{s-1}
    \right]
    &=
    (\varepsilon_s^Z)^2
    -
    \sigma_Z^2
    -
    2\delta_N(x_s)\varepsilon_s^Z.
\end{aligned}
\end{equation}
Both sequences are martingale differences.  The kernel normalization gives
\[
    |f(x_s)|
    \leq
    \|f\|_{\mathcal H_k}
    \sqrt{k(x_s,x_s)}
    \leq B,
\]
and the offline event gives \(|\delta_N(x_s)|\leq e_N\).

The paired fluctuations are conditionally Gaussian with the fixed covariance matrix \(\Sigma\). Hence, if \(X_s\) denotes either centered increment in \eqref{eq:appendix-centered-moment-increments}, there exist constants \(c_0,C_1>0\), depending only on \(B\), \(e_N\), and \(\Sigma\), such that
\begin{equation}
\label{eq:appendix-conditional-mgf-bound}
    \E\!\left[
        \exp(\theta X_s)
        \mid\F_{s-1}
    \right]
    \leq
    \exp(C_1\theta^2),
    \qquad
    |\theta|\leq c_0.
\end{equation}
The Bernstein inequality for martingale differences therefore gives, for every \(u>0\),
\[
    \P\!\left(
        \left|
            \frac{1}{t}\sum_{s=1}^t X_s
        \right|
        >
        C_0
        \left\{
            \sqrt{\frac{u}{t}}
            +
            \frac{u}{t}
        \right\}
    \right)
    \leq
    2e^{-u},
\]
after increasing \(C_0\) if needed. Taking \(u=\log(4T/\alpha)\) and applying a union bound over the two moments and \(t=1,\ldots,T\) proves the result.
\end{proof}

\begin{proof}[Proof of
Theorem~\ref{theorem:online-coefficient-concentration}]
By \eqref{eq:coefficient-conditional-moments},
\begin{equation}
\label{eq:appendix-empirical-moment-representations}
\begin{aligned}
    \frac{1}{t}\sum_{s=1}^t
    Y_s\widetilde Z_s
    &=
    \sigma_{YZ}
    -
    \frac{1}{t}\sum_{s=1}^t
    f(x_s)\delta_N(x_s)
    +
    \Delta_{1,t},\\
    \frac{1}{t}\sum_{s=1}^t
    \widetilde Z_s^2
    &=
    \sigma_Z^2
    +
    \frac{1}{t}\sum_{s=1}^t
    \delta_N^2(x_s)
    +
    \Delta_{2,t}.
\end{aligned}
\end{equation}
There exists a constant \(C_2\geq1\) such that, whenever
\[
    t
    \geq
    C_2\log(4T/\alpha),
\]
the event in Lemma~\ref{lemma:online-moment-concentration} implies
\begin{equation}
\label{eq:appendix-empirical-denominator-lower-bound}
    \frac{1}{t}\sum_{s=1}^t
    \widetilde Z_s^2
    \geq
    \sigma_Z^2-b_t
    \geq
    \frac{\sigma_Z^2}{2}.
\end{equation}
Moreover,
\begin{equation}
\label{eq:appendix-population-coefficient-bound}
    |\lambda_{N,t}|
    \leq
    \frac{
        |\sigma_{YZ}|+Be_N
    }{
        \sigma_Z^2
    }.
\end{equation}

Using \eqref{eq:def-online-coefficient},
\eqref{eq:def-population-coefficient}, and
\eqref{eq:appendix-empirical-moment-representations}, we obtain
\begin{equation}
\label{eq:appendix-ratio-identity}
    \widehat\lambda_t-\lambda_{N,t}
    =
    \frac{
        \Delta_{1,t}-\lambda_{N,t}\Delta_{2,t}
    }{
        t^{-1}\sum_{s=1}^t\widetilde Z_s^2
    }.
\end{equation}
Combining
\eqref{eq:appendix-online-moment-concentration}--%
\eqref{eq:appendix-population-coefficient-bound} gives
\[
    |\widehat\lambda_t-\lambda_{N,t}|
    \leq
    \frac{2b_t}{\sigma_Z^2}
    \left\{
        1+
        \frac{|\sigma_{YZ}|+Be_N}{\sigma_Z^2}
    \right\}.
\]
For \(t\geq C_2\log(4T/\alpha)\), the second term in \(b_t\) is bounded by its square-root term.  Since \(T\geq2\) and \(\alpha\in(0,1)\),
\[
    \log(4T/\alpha)
    \leq
    3\log(T/\alpha).
\]
Increasing \(C\) if needed therefore gives
\[
    |\widehat\lambda_t-\lambda_{N,t}|
    \leq
    C
    \sqrt{
        \frac{\log(T/\alpha)}{t}
    },
    \qquad
    C\log(T/\alpha)\leq t\leq T.
\]
This proves the theorem.
\end{proof}

\begin{proof}[Proof of
Corollary~\ref{corollary:online-coefficient-oracle-gap}] The definitions of \(\lambda_{N,t}\) and \(\lambda^\star=\sigma_{YZ}/\sigma_Z^2\) give
\begin{equation}
\label{eq:appendix-population-oracle-gap}
    \lambda_{N,t}-\lambda^\star
    =
    -
    \frac{
        \sigma_Z^2
        t^{-1}\sum_{s=1}^t
        f(x_s)\delta_N(x_s)
        +
        \sigma_{YZ}
        t^{-1}\sum_{s=1}^t
        \delta_N^2(x_s)
    }{
        \sigma_Z^2
        \left\{
            \sigma_Z^2
            +
            t^{-1}\sum_{s=1}^t
            \delta_N^2(x_s)
        \right\}
    }.
\end{equation}
Since \(|f(x_s)|\leq B\) and
\(|\delta_N(x_s)|\leq e_N\),
\[
    |\lambda_{N,t}-\lambda^\star|
    \leq
    \frac{Be_N}{\sigma_Z^2}
    +
    \frac{|\sigma_{YZ}|e_N^2}{\sigma_Z^4}
    =
    q_N.
\]
The triangle inequality and
Theorem~\ref{theorem:online-coefficient-concentration} give
\[
    |\widehat\lambda_t-\lambda^\star|
    \leq
    a_t+q_N.
\]

Finally, completing the square in
\(\eta_\lambda^2\) gives
\[
    \eta_\lambda^2
    =
    \eta_\star^2
    +
    \sigma_Z^2(\lambda-\lambda^\star)^2.
\]
Setting \(\lambda=\widehat\lambda_t\) yields
\[
    0
    \leq
    \eta_{\widehat\lambda_t}^2-\eta_\star^2
    \leq
    \sigma_Z^2(a_t+q_N)^2.
\]
\end{proof}

\paragraph{Reward transformations.}
We next examine how shifting and rescaling the observed rewards affect coefficient estimation. For fixed constants \(a>0\) and \(c\in\mathbb R\), set \(Y_s^{[a,c]}:=aY_s+c\) on the same observation history, leaving the auxiliary observations and \(\widehat g_N\) unchanged. By \eqref{eq:def-online-coefficient} and \eqref{eq:def-population-coefficient}, the transformed empirical coefficient and its population counterpart satisfy
\begin{equation}
\label{eq:appendix-response-transformation-identities}
\begin{aligned}
    \widehat\lambda_t^{(a,c)}
    &=
    a\widehat\lambda_t
    +
    c\frac{
        \sum_{s=1}^t\widetilde Z_s
    }{
        \sum_{s=1}^t\widetilde Z_s^2
    },\\
    \lambda_{N,t}^{[a,c]}
    &=
    a\lambda_{N,t}
    -
    c\frac{
        t^{-1}\sum_{s=1}^t\delta_N(x_s)
    }{
        \sigma_Z^2+t^{-1}\sum_{s=1}^t\delta_N^2(x_s)
    }.
\end{aligned}
\end{equation}
The corresponding oracle coefficient becomes \(a\lambda^\star\).

\begin{lemma}[Transformed coefficient error]
\label{lemma:response-transformation}
For any \(T\geq2\) and \(\alpha\in(0,1)\), there exists a constant \(C>0\), depending only on \(B\), \(e_N\), and \(\Sigma\), such that, with probability at least \(1-\alpha\), the bound
\begin{equation}
\label{eq:appendix-response-transformation-bound}
    \left|
        \widehat\lambda_t^{(a,c)}-a\lambda^\star
    \right|
    \leq
    (a+|c|)C\sqrt{\frac{\log(T/\alpha)}{t}}
    +a q_N
    +\frac{|c|e_N}{\sigma_Z^2}
\end{equation}
holds simultaneously for \(C\log(T/\alpha)\leq t\leq T\), where \(q_N\) is defined in \eqref{eq:def-offline-coefficient-error}.
\end{lemma}

\begin{proof}
The proof of Theorem~\ref{theorem:online-coefficient-concentration} uses objective regularity only through \(\|f\|_\infty\leq B\). The same argument applies to \(Y_s+1\), since
\(\|f+1\|_\infty\leq B+1\). Applying that argument to \(Y_s\) and \(Y_s+1\), with failure probability \(\alpha/2\) for each, gives a constant \(C_0>0\) such that, with probability at least \(1-\alpha\),
\[
    \max\left\{
        |\widehat\lambda_t-\lambda_{N,t}|,
        |\widehat\lambda_t^{(1,1)}-\lambda_{N,t}^{(1,1)}|
    \right\}
    \leq
    C_0\sqrt{\frac{\log(T/\alpha)}{t}}
\]
simultaneously for \(C_0\log(T/\alpha)\leq t\leq T\).

By \eqref{eq:appendix-response-transformation-identities}, the population shift satisfies
\[
    |\lambda_{N,t}^{(1,1)}-\lambda_{N,t}|
    =
    \frac{
        \left|t^{-1}\sum_{s=1}^t\delta_N(x_s)\right|
    }{
        \sigma_Z^2+t^{-1}\sum_{s=1}^t\delta_N^2(x_s)
    }
    \leq
    \frac{e_N}{\sigma_Z^2}.
\]
Combining this bound with the two concentration bounds and applying the triangle inequality gives
\begin{equation}
\label{eq:appendix-unit-shift-bound}
\begin{aligned}
    |\widehat\lambda_t^{(1,1)}-\widehat\lambda_t|
    &\leq
    |\widehat\lambda_t^{(1,1)}-\lambda_{N,t}^{(1,1)}|
    +|\lambda_{N,t}^{(1,1)}-\lambda_{N,t}|
    +|\lambda_{N,t}-\widehat\lambda_t|\\
    &\leq
    2C_0\sqrt{\frac{\log(T/\alpha)}{t}}
    +\frac{e_N}{\sigma_Z^2}.
\end{aligned}
\end{equation}
The empirical identity in \eqref{eq:appendix-response-transformation-identities} also gives
\[
    \widehat\lambda_t^{(a,c)}
    =
    a\widehat\lambda_t
    +c\bigl(
        \widehat\lambda_t^{(1,1)}-\widehat\lambda_t
    \bigr).
\]
Using \(|\lambda_{N,t}-\lambda^\star|\leq q_N\), we therefore obtain
\[
\begin{aligned}
    |\widehat\lambda_t^{(a,c)}-a\lambda^\star|
    &\leq
    a|\widehat\lambda_t-\lambda^\star|
    +|c|\,|\widehat\lambda_t^{(1,1)}-\widehat\lambda_t|\\
    &\leq
    (a+2|c|)C_0\sqrt{\frac{\log(T/\alpha)}{t}}
    +a q_N+\frac{|c|e_N}{\sigma_Z^2}.
\end{aligned}
\]
Taking \(C\geq2C_0\) and restricting to \(C\log(T/\alpha)\leq t\leq T\) proves \eqref{eq:appendix-response-transformation-bound}.
\end{proof}

On a fixed history, positive scaling rescales the estimated and oracle coefficients equally. Translation adds a sampling term that decreases with online sample size and a population shift bounded by \(|c|e_N/\sigma_Z^2\). Thus, for fixed \(a,c\), \eqref{eq:appendix-response-transformation-bound} retains the original convergence orders in online sample size and offline estimation error.

If only the objective mean is shifted or rescaled, the same coefficient error analysis applies with \(B\) replaced by \(aB+|c|\), while \(\Sigma\) and \(\lambda^\star\) remain unchanged.

\subsection{GP Confidence Bound}
\label{appendix:confidence-corrected-history}

Since \(\widehat\lambda_m\) is estimated from the first \(m\) observations and then applied to them, the confidence argument must hold uniformly over \(\mathcal I_m\).

Let \(\phi:\X\to\mathcal H_k\) be the canonical feature map of \(k\).  For \(n\geq1\), define
\begin{equation}
\label{eq:appendix-feature-covariance-operator}
    \mathbf V_n
    :=
    \overline\eta_m^2\mathbf I_{\mathcal H_k}
    +
    \sum_{s=1}^n
    \phi(x_s)\otimes\phi(x_s).
\end{equation}

\begin{lemma}[Uniform corrected-noise bound]
\label{lemma:uniform-corrected-noise}
Suppose Assumption~\ref{assumption:paired-fluctuations} holds.  With
probability at least \(1-\alpha/2\),
\begin{equation}
\label{eq:appendix-uniform-corrected-noise}
    \frac{1}{\overline\eta_m}
    \left\|
        \sum_{s=1}^n
        \phi(x_s)\varepsilon_s^\lambda
    \right\|_{\mathbf V_n^{-1}}
    \leq
    2
    \sqrt{
        \gamma_n\!\left(k,\overline\eta_m^2\right)
        +
        1
        +
        \log\!\left(\frac{4}{\alpha}\right)
    }
\end{equation}
holds simultaneously for every \(n=1,\ldots,T\) and
\(\lambda\in\mathcal I_m\).
\end{lemma}

\begin{proof}
For every \(\lambda\), decompose the corrected noise into the oracle residual and an auxiliary-noise term:
\[
    \varepsilon_s^\lambda
    =
    \left(
        \varepsilon_s^Y-\lambda^\star\varepsilon_s^Z
    \right)
    -
    \left(
        \lambda-\lambda^\star
    \right)
    \varepsilon_s^Z.
\]
Conditional on \(\F_{s-1}\),
\(\varepsilon_s^Y-\lambda^\star\varepsilon_s^Z\) and
\(\varepsilon_s^Z\) are centered Gaussian with standard deviations
\(\eta_\star\) and \(\sigma_Z\), respectively.

Applying the kernel self-normalized inequality
\citep{chowdhury2017kernelized} to the rescaled feature map
\(\phi/\overline\eta_m\), separately for these two noise sequences, gives an
event with probability at least \(1-\alpha/2\) on which
\[
\begin{aligned}
    \frac{1}{\eta_\star}
    \left\|
        \sum_{s=1}^n
        \phi(x_s)
        \left(
            \varepsilon_s^Y-\lambda^\star\varepsilon_s^Z
        \right)
    \right\|_{\mathbf V_n^{-1}}
    &\leq
    \sqrt{
        2\left\{
            \gamma_n\!\left(k,\overline\eta_m^2\right)
            +
            1
            +
            \log\!\left(\frac{4}{\alpha}\right)
        \right\}
    },\\
    \frac{1}{\sigma_Z}
    \left\|
        \sum_{s=1}^n
        \phi(x_s)\varepsilon_s^Z
    \right\|_{\mathbf V_n^{-1}}
    &\leq
    \sqrt{
        2\left\{
            \gamma_n\!\left(k,\overline\eta_m^2\right)
            +
            1
            +
            \log\!\left(\frac{4}{\alpha}\right)
        \right\}
    }
\end{aligned}
\]
simultaneously for every \(n=1,\ldots,T\).

For \(\lambda\in\mathcal I_m\),
\(|\lambda-\lambda^\star|\leq d_m\).  Moreover,
\[
    \eta_\star+d_m\sigma_Z
    \leq
    \sqrt{
        2\left(
            \eta_\star^2+\sigma_Z^2d_m^2
        \right)
    }
    =
    \sqrt{2}\,\overline\eta_m.
\]
The triangle inequality therefore gives
\eqref{eq:appendix-uniform-corrected-noise}.
\end{proof}

\begin{proof}[Proof of
Theorem~\ref{theorem:confidence-online-correction}]
Fix \(\lambda\in\mathcal I_m\).  Let \(\mu_n^\lambda\) be the posterior
mean formed from \(\{(x_s,\widetilde Y_s(\lambda))\}_{s=1}^n\) using
\(\overline\eta_m^2\) as the noise parameter.  The corresponding posterior
standard deviation satisfies
\[
    s_n(x)
    =
    \overline\eta_m
    \|\phi(x)\|_{\mathbf V_n^{-1}}
\]
and does not depend on \(\lambda\). By
\eqref{eq:appendix-fixed-correction-model} and
\eqref{eq:offline-rkhs-error-bound},
\begin{equation}
\label{eq:appendix-fixed-target-rkhs-bound}
    \|h_\lambda\|_{\mathcal H_k}
    \leq
    B+|\lambda|D.
\end{equation}
The standard kernel-ridge decomposition gives
\begin{equation}
\label{eq:appendix-kernel-confidence-decomposition}
\begin{aligned}
    |h_\lambda(x)-\mu_n^\lambda(x)|
    \leq
    \left\{
        \|h_\lambda\|_{\mathcal H_k}
        +
        \frac{1}{\overline\eta_m}
        \left\|
            \sum_{s=1}^n
            \phi(x_s)\varepsilon_s^\lambda
        \right\|_{\mathbf V_n^{-1}}
    \right\}
    s_n(x).
\end{aligned}
\end{equation}
On the event in Lemma~\ref{lemma:uniform-corrected-noise},
\eqref{eq:appendix-fixed-target-rkhs-bound} and
\eqref{eq:appendix-kernel-confidence-decomposition} imply
\begin{equation}
\label{eq:appendix-uniform-fixed-coefficient-confidence}
\begin{aligned}
    |h_\lambda(x)-\mu_n^\lambda(x)|
    \leq
    \Bigg[
        B+|\lambda|D
        +
        2
        \sqrt{
            \gamma_n\!\left(k,\overline\eta_m^2\right)
            +
            1
            +
            \log\!\left(\frac{4}{\alpha}\right)
        }
    \Bigg]
    s_n(x)
\end{aligned}
\end{equation}
simultaneously for every \(n=1,\ldots,T\), \(x\in\X\), and
\(\lambda\in\mathcal I_m\).

By the construction in Subsection~\ref{appendix:proof-setup},
\(\widehat\lambda_m\in\mathcal I_m\) with probability at least
\(1-\alpha/2\). The intersection of this event with the event in Lemma~\ref{lemma:uniform-corrected-noise} has probability at least
\(1-\alpha\). On this intersection, set
\(\lambda=\widehat\lambda_m\) and \(n=t-1\).  Then
\[
    h_\lambda=h_m,
    \qquad
    \mu_n^\lambda=\mu_{t-1},
\]
and the right-hand side of
\eqref{eq:appendix-uniform-fixed-coefficient-confidence} equals
\(\sqrt{\beta_t}\,s_{t-1}(x)\).  This proves
\eqref{eq:confidence-stabilized-target} simultaneously for every
\(t=m+1,\ldots,T\) and \(x\in\X\).

The confidence event is uniform over \(\mathcal I_m\).  Hence,
\(\widehat\lambda_m\) may be estimated from the first \(m\) observations and
then applied to the full corrected history.
\end{proof}

\subsection{Cumulative Regret}
\label{appendix:cumulative-regret}

\begin{proof}[Proof of
Lemma~\ref{lemma:cumulative-posterior-width}]
The determinant identity for sequential GP updates gives
\[
    \sum_{t=m+1}^T
    \log\!\left(
        1+
        \frac{
            s_{t-1}^2(x_t)
        }{
            \overline\eta_m^2
        }
    \right)
    \leq
    \log\det\!\left(
        \mathbf I_T
        +
        \overline\eta_m^{-2}\mathbf K_T
    \right)
    \leq
    2\gamma_T\!\left(k,\overline\eta_m^2\right).
\]
The kernel normalization implies
\(s_{t-1}(x_t)\leq1\).  For every \(t>m\),
\[
\begin{aligned}
    s_{t-1}(x_t)
    \leq{}&
    \overline\eta_m
    \sqrt{
        2
        \log\!\left(
            1+
            \frac{
                s_{t-1}^2(x_t)
            }{
                \overline\eta_m^2
            }
        \right)
    }\\
    &+
    \frac{1}{\log 2}
    \log\!\left(
        1+
        \frac{
            s_{t-1}^2(x_t)
        }{
            \overline\eta_m^2
        }
    \right).
\end{aligned}
\]
When \(s_{t-1}(x_t)\leq\overline\eta_m\), the first term gives the
inequality by \(\log(1+z)\geq z/2\) for \(z\in[0,1]\).  When
\(s_{t-1}(x_t)>\overline\eta_m\), the second term is at least one and
therefore bounds \(s_{t-1}(x_t)\).

Summing this inequality and applying the Cauchy--Schwarz inequality gives
\[
\begin{aligned}
    \sum_{t=m+1}^T s_{t-1}(x_t)
    \leq{}&
    2\overline\eta_m
    \sqrt{
        (T-m)
        \gamma_T\!\left(k,\overline\eta_m^2\right)
    }\\
    &+
    \frac{2}{\log 2}
    \gamma_T\!\left(k,\overline\eta_m^2\right).
\end{aligned}
\]
This proves the lemma.
\end{proof}

When \(\overline\eta_m\) converges to a positive constant and
\(\gamma_T=o(T-m)\), the second term is lower order.  This regime is considered in the Matérn specialization below.

\begin{proof}[Proof of
Theorem~\ref{theorem:cumulative-regret}]
On the event in
Theorem~\ref{theorem:confidence-online-correction},
Assumption~\ref{assumption:kernel-objective} implies
\(|f(x)|\leq B\) for every \(x\in\X\).  Hence,
\[
    \sum_{t=1}^m r_t
    \leq
    2Bm.
\]
For every \(t>m\), the definition of \(h_m\) gives
\[
\begin{aligned}
    r_t
    &=
    h_m(x^\star)-h_m(x_t)
    +
    \widehat\lambda_m
    \left\{
        \delta_N(x_t)-\delta_N(x^\star)
    \right\}\\
    &\leq
    2\sqrt{\beta_t}\,s_{t-1}(x_t)
    +
    |\widehat\lambda_m|
    |\delta_N(x_t)-\delta_N(x^\star)|.
\end{aligned}
\]
The inequality follows from the confidence bound and the GP-UCB selection
rule.  Since \(\beta_t\) is nondecreasing, summing over the post-\(m\)
rounds yields
\begin{equation}
\label{eq:appendix-pathwise-regret-bound}
\begin{aligned}
    R_T
    \leq{}&
    2Bm
    +
    2\sqrt{\beta_T}
    \sum_{t=m+1}^T
    s_{t-1}(x_t)\\
    &+
    |\widehat\lambda_m|
    \sum_{t=m+1}^T
    |\delta_N(x_t)-\delta_N(x^\star)|.
\end{aligned}
\end{equation}
The transfer term depends on differences in the offline error between \(x_t\) and \(x^\star\). The uniform error bound gives
\[
    |\delta_N(x_t)-\delta_N(x^\star)|
    \leq
    2e_N.
\]
Combining Lemma~\ref{lemma:cumulative-posterior-width} with \eqref{eq:appendix-pathwise-regret-bound} gives \eqref{eq:general-cumulative-regret} for an absolute constant \(C\). Since the confidence event has probability at least \(1-\alpha\), the result follows.
\end{proof}

\subsection{Kernel-Specific Rates}
\label{appendix:kernel-specific-rates}

Write
\[
    p:=\frac{d}{2\nu+d}.
\]
The condition \(\nu>d/2\) implies \(p<1/2\).

\paragraph{Information gain.}
We bound information gain by decomposing each kernel matrix into a low-rank approximation and a small remainder. Let \(k\) be the restriction of a stationary Matérn kernel on \(\mathbb R^d\), with fixed hyperparameters. For any query sequence \(A=(x_1,\ldots,x_n)\in\X^n\), write
\[
    \mathbf K_A
    :=
    [k(x_i,x_j)]_{i,j=1}^n.
\]

Fix a closed ball \(\Omega\) containing \(\X\), and use the same kernel on \(\Omega\). For each sufficiently large integer \(r\), choose distinct sites \(z_1,\ldots,z_r\in\Omega\) such that every point of \(\Omega\) lies within \(C_1r^{-1/d}\) of a site, for a fixed constant \(C_1\). Define \(\mathbf C_r\in\mathbb R^{n\times r}\) and \(\mathbf W_r\in\mathbb R^{r\times r}\) by
\[
    (\mathbf C_r)_{ij}:=k(x_i,z_j),
    \qquad
    (\mathbf W_r)_{ij}:=k(z_i,z_j).
\]
The interpolation approximation and its remainder are
\[
    \mathbf L_r
    :=
    \mathbf C_r\mathbf W_r^{-1}\mathbf C_r^\top,
    \qquad
    \mathbf E_r
    :=
    \mathbf K_A-\mathbf L_r.
\]
The matrix \(\mathbf L_r\) is positive semidefinite, and the Schur complement gives \(\mathbf E_r\succeq0\). The uniform Matérn interpolation error bound \citep[Section~5]{santin2016approximation} gives
\begin{equation}
\label{eq:appendix-matern-interpolation-residual}
    \max_{1\leq i\leq n}(\mathbf E_r)_{ii}
    \leq C_0r^{-2\nu/d},
    \qquad r\geq r_0,
\end{equation}
where \(C_0\) and \(r_0\) are independent of the query sequence. Consequently, using \(k(x,x)\leq1\),
\[
    \operatorname{rank}(\mathbf L_r)\leq r,
    \qquad
    \operatorname{tr}(\mathbf L_r)\leq n,
    \qquad
    \operatorname{tr}(\mathbf E_r)
    \leq C_0nr^{-2\nu/d}.
\]
Factoring the determinant and applying \(\log\det(\mathbf I+\mathbf H)\leq\operatorname{tr}(\mathbf H)\) for \(\mathbf H\succeq0\), together with the rank and trace bounds, yields
\begin{equation}
\label{eq:appendix-matern-logdet-bound}
\begin{aligned}
    \log\det\!\left(\mathbf I_n+\eta^{-2}\mathbf K_A\right)
    &\leq
    \log\det\!\left(\mathbf I_n+\eta^{-2}\mathbf L_r\right)
    +\eta^{-2}\operatorname{tr}(\mathbf E_r)\\
    &\leq
    r\log\!\left(1+\frac{n}{r\eta^2}\right)
    +C_0\frac{n}{\eta^2}r^{-2\nu/d}.
\end{aligned}
\end{equation}
Set \(u=n/\eta^2\) and choose
\[
    r
    =
    \max\!\left\{
        r_0,\,
        \left\lceil
            \left(\frac{u}{\log(1+u)}\right)^p
        \right\rceil
    \right\}.
\]
Since \(2\nu/d=(1-p)/p\), substituting this choice into \eqref{eq:appendix-matern-logdet-bound} and maximizing over query sequences gives, for \(n\geq2\) and \(\eta^2>0\),
\begin{equation}
\label{eq:appendix-matern-information-gain}
    \gamma_n(k,\eta^2)
    \leq
    C\left\{
        \left(\frac{n}{\eta^2}\right)^p
        \log^{1-p}\!\left(1+\frac{n}{\eta^2}\right)
        +
        \log\!\left(1+\frac{n}{\eta^2}\right)
    \right\}.
\end{equation}
The constant \(C\) depends only on the fixed ball and kernel parameters, not on \(n\) or \(\eta^2\). This proves \eqref{eq:matern-information-gain} while retaining the dependence on the noise variance.

\begin{proof}[Proof of
Corollary~\ref{corollary:matern-regret-rate}] Consider the intersection of the offline event in Assumption~\ref{assumption:offline-estimation} and the two online events used to establish Theorem~\ref{theorem:cumulative-regret}: coefficient concentration and uniform GP confidence. This event has probability at least \(1-\alpha_N-\alpha\).

Since \(e_N\to0\), the concentration constants can be chosen uniformly in \(T\). At fixed confidence levels and for sufficiently large \(T\), \(m=\lceil\sqrt T\rceil\) lies in the range of Theorem~\ref{theorem:online-coefficient-concentration}, which gives
\[
    a_m
    =
    \widetilde O(T^{-1/4}).
\]
The prescribed offline rate and the definition of \(q_N\) give
\[
    q_N
    =
    O(e_N+e_N^2)
    =
    \widetilde O\!\left(
        T^{-\frac{1-p}{2}}
    \right).
\]
Since \(p<1/2\), these bounds imply \(a_m+q_N=\widetilde O(T^{-1/4})\). Corollary~\ref{corollary:online-coefficient-oracle-gap} and the definition of \(\overline\eta_m^2\) now give
\begin{equation}
\label{eq:appendix-matern-coefficient-rates}
\begin{aligned}
    |\widehat\lambda_m|
    &\leq
    |\lambda^\star|+a_m+q_N
    =
    O(1),\\
    \overline\eta_m^2-\eta_\star^2
    &=
    \sigma_Z^2(a_m+q_N)^2
    =
    \widetilde O(T^{-1/2}).
\end{aligned}
\end{equation}
For fixed \(\Sigma\) with \(\rho\in(-1,1)\), \(\eta_\star^2=\sigma_Y^2(1-\rho^2)>0\). Thus, \eqref{eq:appendix-matern-coefficient-rates} gives
\[
    \overline\eta_m^2
    =
    \sigma_Y^2(1-\rho^2)\{1+o(1)\}.
\]
Applying \eqref{eq:appendix-matern-information-gain} yields
\begin{equation}
\label{eq:appendix-matern-information-gain-rate}
\begin{aligned}
    \gamma_T\!\left(k,\overline\eta_m^2\right)
    &=
    \widetilde O\!\left(
        \overline\eta_m^{-2p}T^p
    \right)\\
    &=
    \widetilde O\!\left(
        (1-\rho^2)^{-p}T^p
    \right).
\end{aligned}
\end{equation}
The definition of \(\beta_T\), together with the bound \(|\widehat\lambda_m|=O(1)\) in \eqref{eq:appendix-matern-coefficient-rates}, gives
\[
    \sqrt{\beta_T}
    =
    O\!\left(
        1+
        \sqrt{
            \gamma_T\!\left(k,\overline\eta_m^2\right)
        }
    \right).
\]
Since \(T-m\asymp T\), substituting these bounds into the second term of \eqref{eq:general-cumulative-regret} gives
\begin{align}
    &\sqrt{\beta_T}
    \left\{
        \overline\eta_m
        \sqrt{
            (T-m)\gamma_T\!\left(k,\overline\eta_m^2\right)
        }
        +
        \gamma_T\!\left(k,\overline\eta_m^2\right)
    \right\}
    \nonumber\\
    &\qquad=
    \widetilde O\!\left(
        \overline\eta_m^{1-2p}T^{\frac12+p}
        +
        \overline\eta_m^{1-p}T^{\frac{1+p}{2}}
        +
        \overline\eta_m^{-2p}T^p
        +
        \overline\eta_m^{-3p}T^{\frac{3p}{2}}
    \right)
    \nonumber\\
    &\qquad=
    \widetilde O\!\left(
        (1-\rho^2)^{\frac12-p}T^{\frac12+p}
        +
        T^{\frac{1+p}{2}}
    \right).
\label{eq:appendix-matern-learning-rate}
\end{align}
Since \(p<1/2\) and \(\overline\eta_m\) converges to the positive constant \(\eta_\star\), the last three terms in the expansion are \(\widetilde O(T^{(1+p)/2})\), with constants that may depend on the fixed noise level
\(\eta_\star\).

The first and third terms in \eqref{eq:general-cumulative-regret} satisfy
\[
    2Bm
    =
    O(T^{1/2}),
    \qquad
    2(T-m)|\widehat\lambda_m|e_N
    =
    \widetilde O\!\left(
        T^{\frac{1+p}{2}}
    \right).
\]
Combining these bounds gives
\[
\begin{aligned}
    R_T
    &=
    \widetilde O\!\left(
        (1-\rho^2)^{\frac12-p}T^{\frac12+p}
        +
        T^{\frac{1+p}{2}}
    \right)\\
    &=
    \widetilde O\!\left(
        (1-\rho^2)^{
            \frac{2\nu-d}{2(2\nu+d)}
        }
        T^{
            \frac{2\nu+3d}{2(2\nu+d)}
        }
        +
        T^{
            \frac{\nu+d}{2\nu+d}
        }
    \right),
\end{aligned}
\]
where the second equality uses \(p=d/(2\nu+d)\). This proves \eqref{eq:matern-cumulative-regret-rate}.
\end{proof}
Taking \(\nu=5/2\) and \(d=2\) gives \eqref{eq:matern-five-halves-two-dimensional-rate}. The simple-regret rate in \eqref{eq:matern-simple-regret-rate} follows from \(S_T\leq R_T/T\).

\paragraph{Offline estimation.}
We construct an offline estimator of the auxiliary mean satisfying both \eqref{eq:offline-uniform-error} and \eqref{eq:offline-rkhs-error-bound}. Suppose that \(\X\) is compact and convex with nonempty interior and
\(g\in\mathcal H_k\). The offline data consist of noisy evaluations of \(g\) at \(N\) quasi-uniformly spaced locations in \(\X\). The observation errors are independent centered Gaussian variables with variance \(\sigma_Z^2\).

We obtain \(\bar g_N\) by Matérn kernel ridge regression, which minimizes the average squared fitting error plus a penalty on the squared RKHS norm. The Matérn RKHS coincides with the Sobolev space of order \(\nu+d/2\), with equivalent norms. With a regularization parameter of order \(N^{-1}\log N\), Corollary~8 and Theorem~8 of \citep{tuo2024asymptotic} give
\begin{equation}
\label{eq:appendix-matern-offline-rate}
    \|\bar g_N-g\|_\infty
    =
    O_{\mathbb P}\!\left(
        \left(\frac{\log N}{N}\right)^{\frac{\nu}{2\nu+d}}
    \right).
\end{equation}

To also satisfy \eqref{eq:offline-rkhs-error-bound}, fix a known bound \(G\geq\|g\|_{\mathcal H_k}\), independent of \(N\) and \(T\). Let \(\widehat g_N\) be a measurable minimizer of the uniform distance to \(\bar g_N\) over RKHS functions with norm at most \(G\):
\begin{equation}
\label{eq:offline-constrained-estimator}
    \widehat g_N
    \in
    \operatorname*{argmin}_{\|h\|_{\mathcal H_k}\leq G}
    \|h-\bar g_N\|_\infty.
\end{equation}
Since \(k\) is continuous and \(\X\) is compact, the feasible RKHS
ball is compact in the uniform norm. The objective is continuous in \(h\) and \(\bar g_N\), so the measurable maximum theorem yields a measurable minimizer. Since \(g\) is feasible, optimality gives \(\|\widehat g_N-\bar g_N\|_\infty \leq\|g-\bar g_N\|_\infty\). The triangle inequality then yields
\begin{equation}
\label{eq:offline-joint-error-bounds}
\begin{aligned}
    \|\widehat g_N-g\|_\infty
    &\leq 2\|\bar g_N-g\|_\infty,\\
    \|\widehat g_N-g\|_{\mathcal H_k}
    &\leq
    \|\widehat g_N\|_{\mathcal H_k}
    +\|g\|_{\mathcal H_k}
    \leq 2G.
\end{aligned}
\end{equation}

Thus, \(\widehat g_N\) attains \eqref{eq:matern-offline-supnorm-rate} and satisfies \eqref{eq:offline-rkhs-error-bound} with \(D=2G\). At a fixed offline confidence level, \(N=\widetilde\Omega(T)\) suffices for \eqref{eq:matern-rate-conditions}. The experiments use the unconstrained GP estimate described in Appendix~\ref{appendix:experimental-setup}.

\section{Additional Experimental Details}
\label{appendix:experimental-details}

This appendix provides the complete experimental protocol for the results in Section~\ref{sec:experiments}.  It also reports additional results and diagnostic checks.

\subsection{Experimental Setup}
\label{appendix:experimental-setup}

\paragraph{Benchmark functions.}
For the cumulative-regret experiments, we consider four standard
two-dimensional test functions with different optimization landscapes.  The low-dimensional setting allows us to study how noise reduction changes the GP sampling path over a long horizon under the continuous acquisition protocol. Ackley is oscillatory and multimodal, Branin has several separated basins, Rosenbrock contains a narrow curved valley, and Six-Hump Camel has multiple local optima. For the simple-regret experiments, we retain Branin and use higher-dimensional Ackley and Rosenbrock instances. We additionally include Hartmann and Levy, giving six problems spanning dimensions \(2\) to \(10\). Hartmann adds nonseparable interactions across coordinates, while Levy provides a higher-dimensional multimodal landscape. Let \(f_{\mathrm{raw}}\) denote the original minimization objective and define
\[
    f_{\mathrm{raw}}^\star
    :=
    \min_{x\in\X} f_{\mathrm{raw}}(x).
\]
We convert each benchmark to the reward function
\begin{equation}
\label{eq:experimental-reward-scaling}
    f(x)
    :=
    -\frac{
        f_{\mathrm{raw}}(x)-f_{\mathrm{raw}}^\star
    }{
        s_f
    },
\end{equation}
where \(s_f>0\) is a fixed scale.  This transformation preserves the optimal locations and gives \(f(x^\star)=0\). For the cumulative-regret experiments, we set \(s_f\) to \(5\), \(20\), \(200\), and \(20\) for Ackley, Branin, Rosenbrock, and Six-Hump Camel, respectively. Simple regret is reported on the original objective scale, corresponding to \(s_f=1\). Table~\ref{tab:experimental-benchmark-functions} reports the dimensions, domains, and global minima.
\begin{table}[!htbp]
    \centering
    \caption{Summary of the benchmark problems.  Global minima are rounded
    to four decimal places.  Detailed function definitions follow
    \citep{surjanovic2013virtual}.}
    \label{tab:experimental-benchmark-functions}
    \begin{tabular*}{\textwidth}{
        @{\extracolsep{\fill}}
        lccc
        @{}
    }
        \toprule
        Task
        & \(d\)
        & Domain \(\X\)
        & \(f_{\mathrm{raw}}^\star\) \\
        \midrule
        \multicolumn{4}{@{}l@{}}{\textbf{Panel A:} Cumulative-regret experiments} \\ \addlinespace[2pt]
        Ackley
        & \(2\)
        & \([-5,5]^2\)
        & \(0\) \\
        Branin
        & \(2\)
        & \([-5,10]\times[0,15]\)
        & \(0.3979\) \\
        Rosenbrock
        & \(2\)
        & \([-2.048,2.048]^2\)
        & \(0\) \\
        Six-Hump Camel
        & \(2\)
        & \([-3,3]\times[-2,2]\)
        & \(-1.0316\) \\
        \midrule
        \multicolumn{4}{@{}l@{}}{\textbf{Panel B:} Simple-regret experiments} \\ \addlinespace[2pt]
        Branin
        & \(2\)
        & \([-5,10]\times[0,15]\)
        & \(0.3979\) \\
        Hartmann
        & \(3\)
        & \([0,1]^3\)
        & \(-3.8628\) \\
        Ackley
        & \(4\)
        & \([-5,5]^4\)
        & \(0\) \\
        Rosenbrock
        & \(5\)
        & \([-5,10]^5\)
        & \(0\) \\
        Hartmann
        & \(6\)
        & \([0,1]^6\)
        & \(-3.3224\) \\
        Levy
        & \(10\)
        & \([-10,10]^{10}\)
        & \(0\) \\
        \bottomrule
    \end{tabular*}
\end{table}

\paragraph{Paired observation model.}
At each online query \(x_t\), we observe
\begin{equation}
\label{eq:experimental-paired-observations}
    Y_t
    =
    f(x_t)+\varepsilon_t^Y,
    \qquad
    Z_t
    =
    g(x_t)+\varepsilon_t^Z,
\end{equation}
where \(f\) is the reward function defined in \eqref{eq:experimental-reward-scaling}.  To isolate observation-level variance reduction from function-level transfer, we construct the shape of \(g\) independently of \(f\) and use the same normalized surface across all tasks.  Write the domain as \(\X=\prod_{j=1}^d[\ell_j,r_j]\), and define
\[
    u_j
    :=
    \frac{x_j-\ell_j}{r_j-\ell_j},
    \qquad
    j=1,\ldots,d.
\]
For \(d=2\), we set \(u_3=0\).  We use the fixed three-term Fourier surface
\begin{equation}
\label{eq:experimental-auxiliary-mean}
\begin{aligned}
    g(x)
    ={}&
    1    +0.45\sin\!\left(2\pi(u_1+0.17)\right)\\
    &+    0.30\cos\!\left(2\pi(1.7u_2+0.11)\right)\\
    &+
    0.20\sin\!\left(
        2\pi(0.9u_1+0.6u_2+0.4u_3)
    \right).
\end{aligned}
\end{equation}
This construction follows the use of sinusoidal features for smooth stationary kernels \citep{rahimi2007random}.  Its relative coefficients, frequencies, and phases are fixed across tasks and correlation levels and are not selected using the benchmark objectives. For the simple-regret
experiments, this surface is multiplied by the standard deviation of \(f_{\mathrm{raw}}\) on a fixed \(2^{14}\)-point Sobol design.

\paragraph{Noise model.}
Following Assumption~\ref{assumption:paired-fluctuations}, the paired errors are jointly Gaussian with zero mean and covariance \(\Sigma\) fixed across query locations and rounds. For the cumulative-regret experiments, we set \(\sigma_Y^2=4.25\) and \(\sigma_Z^2=1.0625\). For the simple-regret experiments, both noise standard deviations are scaled by the same factor as \(g\). Within each task, we vary within-pair correlation while holding both noise variances fixed.

\paragraph{Offline estimation.}
The offline data contain \(q=16\) independent auxiliary observations at each of \(M=1024\) Sobol locations, giving \(N=Mq=16384\) observations in total. We fit a Matérn-\(5/2\) GP on the normalized domain using the sample average at each Sobol location.  Its constant mean is estimated from the offline data, and its noise variance is set to \(\sigma_Z^2/q\). Its kernel parameters are selected by marginal likelihood before optimization, and the resulting estimate \(\widehat g_N\) remains fixed throughout the online run.

\paragraph{Online optimization.}
The cumulative-regret experiments use \(T=1000\) expensive evaluations including initialization. The simple-regret experiments use \(20\) initial queries followed by \(300\) adaptive evaluations. All experiments use a Matérn-\(5/2\) GP with lengthscale \(0.2\) and an empirical constant mean. We use unit output scale, rescaled to the objective units for the simple-regret experiments.

For AI-GPB, the empirical mean is computed from the corrected history, giving the posterior mean
\begin{equation}
\label{eq:experimental-gp-mean}
    \mu_{t-1}(x)
    =
    \bar y_{t-1}
    +
    \mathbf k_{t-1}(x)^\top
    \left(
        \mathbf K_{t-1}
        +\widehat\eta_{t-1}^2\mathbf I_{t-1}
    \right)^{-1}
    \left(
        \widetilde{\mathbf Y}_{t-1}^{(t)}
        -\bar y_{t-1}\mathbf 1
    \right),
\end{equation}
where \(\bar y_{t-1}:=(t-1)^{-1}\sum_{s=1}^{t-1}\widetilde Y_s^{(t)}\). The posterior variance is computed as in \eqref{eq:corrected-gp-posterior}.

All methods use the GP-UCB acquisition rule in \eqref{eq:experimental-acquisition}, with default exploration multiplier \(b=0.35\). The acquisition function is optimized by multistart L-BFGS-B initialized at Sobol points. The same optimization settings are used across methods within each task.

Results for the main experiments are averaged over \(20\) independent replications. Shaded bands indicate one standard error of the mean. Within each replication, all methods share the initial queries. We report both cumulative and simple regret across benchmarks.

\subsection{Additional Experiments}
\label{appendix:additional-experiments}
\paragraph{Online coefficient estimation.}
We examine how closely the coefficient learned by AI-GPB approaches \(\lambda^\star\) during online optimization. We use the four cumulative-regret tasks with \(\rho^2=0.8\) and \(M=1024\) offline locations. Figure~\ref{fig:coefficient-rmse} reports the root mean squared error across \(20\) independent runs at each online iteration,
\begin{equation}
\label{eq:rmse_def}
    \operatorname{RMSE}
    :=
    \sqrt{
        \frac{1}{20}
        \sum_{i=1}^{20}
        \left(\widehat\lambda^{(i)}-\lambda^\star\right)^2
    },
\end{equation}
where \(\widehat\lambda^{(i)}\) is the estimate from run \(i\).

\begin{figure}[!htbp]
    \centering
    \includegraphics[width=0.6\textwidth]{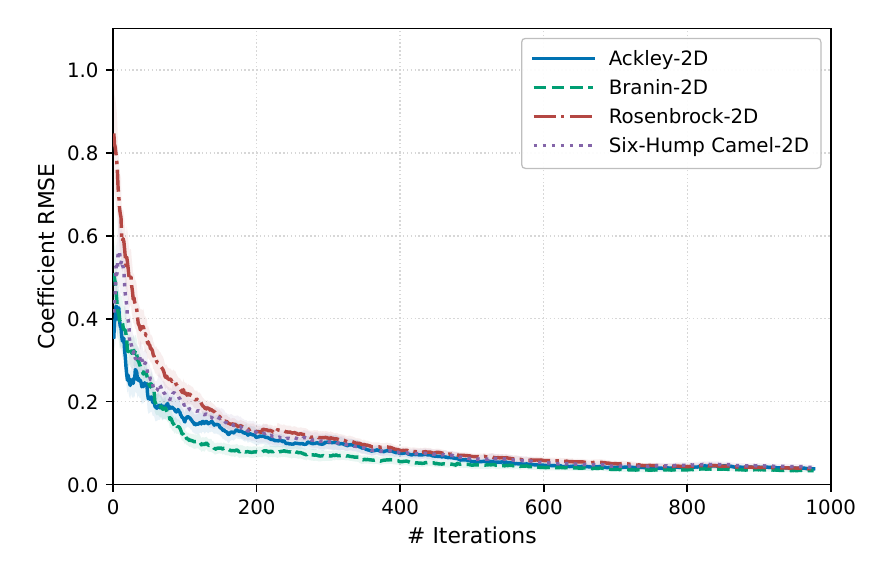}
    \caption{\textbf{Online coefficient estimation at \(\rho^2=0.8\).} Curves report RMSE relative to \(\lambda^\star\). Shading denotes one pointwise bootstrap standard error.}
    \label{fig:coefficient-rmse}
\end{figure}

Across all four tasks, coefficient estimation becomes more accurate during the online run. At the terminal horizon, the RMSE ranges from \(0.033\) to \(0.041\). By \eqref{eq:online-coefficient-oracle-gap}, the average excess noise variance is proportional to the squared RMSE, so the learned correction yields an average noise variance close to the oracle minimum.

\paragraph{Offline sample size.}
Figure~\ref{fig:offline-sample-size} shows how the offline sample size affects estimation accuracy and optimization performance. In panel (a), increasing the number of offline locations \(M\) reduces the maximum estimation error on a fixed evaluation grid, improving the centering used for online correction. Panel (b) reports the Ackley results, showing larger regret reductions at higher correlations and more consistent performance across runs with larger offline budgets. At \(\rho^2=0.8\), \(M=128\) already reduces mean cumulative regret by approximately \(43\%\) relative to Vanilla GP-UCB. Further increases in the offline budget bring mean regret close to the known-\(g\) reference.

\begin{figure}[!htbp]
    \centering
    \includegraphics[width=\textwidth]{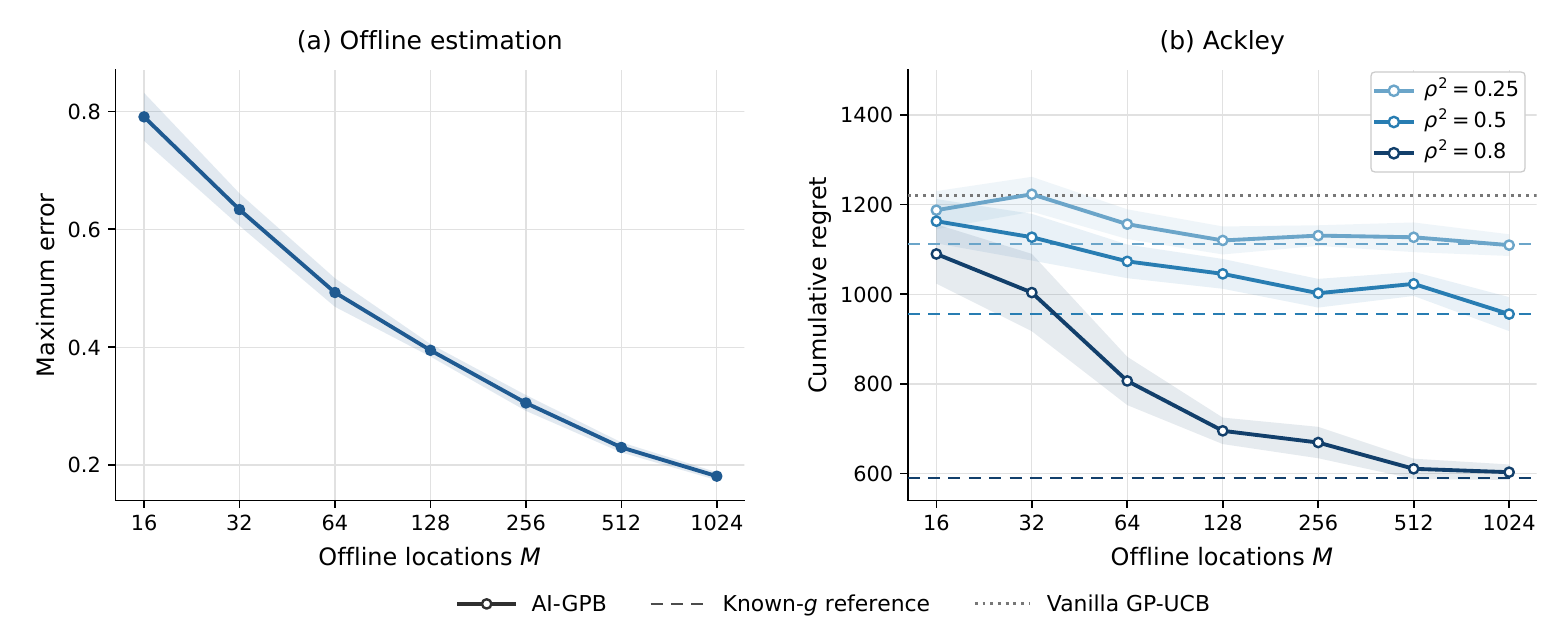}
    \caption{\textbf{Effect of offline sample size.}
    Curves report means over 20 replications, with \(N=16M\) and \(T=1000\). Shading denotes one standard error of the mean. Dashed lines show the known-\(g\) references at the corresponding correlation levels.}
    \label{fig:offline-sample-size}
\end{figure}

Table~\ref{tab:offline-sample-size} compares performance across all four tasks under different offline budgets. At \(M=128\), AI-GPB reduces mean cumulative regret by
\(35\%\)--\(50\%\) relative to Vanilla GP-UCB. The benefit of additional data varies across tasks, with Branin and Six-Hump Camel attaining their lowest mean regret at \(M=16\). The correction in \eqref{eq:def-corrected-outcome-estimated-g} links offline estimation to adaptive query selection, so the optimization benefit of additional offline data depends on both centering accuracy and the resulting search path.

\begin{table}[!htbp]
    \centering
    \caption{\textbf{Offline sample size at \(\rho^2=0.8\).}
    Entries report the ratio of mean cumulative regret to that of Vanilla GP-UCB over \(20\) independent runs.}
    \label{tab:offline-sample-size}
    \begin{tabular*}{\textwidth}{
        @{\extracolsep{\fill}}
        lcccc
        @{}
    }
        \toprule
        & \multicolumn{3}{c}{AI-GPB} & \\
        \cmidrule(lr){2-4}
        Task
        & \(M=16\)
        & \(M=128\)
        & \(M=1024\)
        & Known-\(g\) \\
        \midrule
        Ackley         & 0.894 & 0.570 & 0.494 & 0.483 \\
        Branin         & 0.444 & 0.497 & 0.597 & 0.572 \\
        Rosenbrock     & 0.698 & 0.573 & 0.594 & 0.564 \\
        Six-Hump Camel & 0.571 & 0.653 & 0.652 & 0.638 \\
        \bottomrule
    \end{tabular*}
\end{table}

\paragraph{Exploration sensitivity.}
We examine sensitivity to the GP-UCB exploration parameter, whose theoretical choice can be conservative in practice \citep{takeno2023randomized}. On the four tasks in Section~\ref{subsec:experimental-cumulative-regret}, we implement \eqref{eq:ppi-gp-ucb-acquisition} with
\(\beta_t=b\beta_t^{\mathrm{base}}\) and vary \(b\in\{0.10,0.35,0.50\}\):
\begin{equation}
\label{eq:experimental-acquisition}
    x_t
    \in
    \arg\max_{x\in\X}
    \left\{
        \mu_{t-1}(x)
        +
        \sqrt{b\beta_t^{\mathrm{base}}}\,s_{t-1}(x)
    \right\}.
\end{equation}
Here, \(\beta_t^{\mathrm{base}} =2\log\!\bigl(C\pi^2(t+1)^2/(6\delta)\bigr)\) is the baseline exploration schedule, with \(C=512\) and \(\delta=0.1\). Table~\ref{tab:exploration-sensitivity} reports the ratio of mean cumulative regret for AI-GPB to that for Vanilla GP-UCB at the same \(b\).

Across all four tasks and three exploration multipliers, AI-GPB achieves lower mean cumulative regret than Vanilla GP-UCB. At \(\rho^2=0.8\), the reductions range from \(28.1\%\) to \(52.6\%\), while gains at weaker correlation vary across tasks. The gains are therefore not specific to the default exploration multiplier.
\begin{table}[!htbp]
    \centering
    \caption{\textbf{Exploration sensitivity.}
    Entries report mean cumulative regret of AI-GPB relative to Vanilla GP-UCB at the same \(b\), based on \(10\) independent runs.}
    \label{tab:exploration-sensitivity}
    \begin{tabular*}{\textwidth}{
        @{\extracolsep{\fill}}
        lcccc
        @{}
    }
        \toprule
        & & \multicolumn{3}{c}{\(\rho^2\)} \\
        \cmidrule(lr){3-5}
        Task & \(b\) & \(0.25\) & \(0.5\) & \(0.8\) \\
        \midrule
        Ackley
        & \(0.10\) & 0.951 & 0.923 & 0.601 \\
        & \(0.35\) & \textbf{0.888} & \textbf{0.739} & \textbf{0.474} \\
        & \(0.50\) & 0.913 & 0.817 & 0.540 \\
        \addlinespace[4pt]
        Branin
        & \(0.10\) & \textbf{0.885} & \textbf{0.722} & 0.605 \\
        & \(0.35\) & 0.890 & 0.798 & \textbf{0.571} \\
        & \(0.50\) & 0.897 & 0.776 & 0.589 \\
        \addlinespace[4pt]
        Rosenbrock
        & \(0.10\) & \textbf{0.798} & 0.724 & \textbf{0.554} \\
        & \(0.35\) & 0.813 & \textbf{0.708} & 0.580 \\
        & \(0.50\) & 0.826 & 0.744 & 0.570 \\
        \addlinespace[4pt]
        Six-Hump Camel
        & \(0.10\) & 0.947 & 0.837 & 0.719 \\
        & \(0.35\) & 0.938 & \textbf{0.764} & 0.653 \\
        & \(0.50\) & \textbf{0.869} & 0.810 & \textbf{0.600} \\
        \bottomrule
    \end{tabular*}
\end{table}

\end{document}